\documentclass{article}

\newif\ifarxiv
\arxivfalse

\ifarxiv
  \usepackage[preprint]{neurips_2026}
\else
  \usepackage[main, final]{neurips_2026}
\fi

\usepackage[utf8]{inputenc} 
\usepackage[T1]{fontenc}    
\usepackage{url}            
\usepackage{booktabs}       
\usepackage{amsfonts}       
\usepackage{nicefrac}       
\usepackage{microtype}      
\usepackage[table]{xcolor}  
\usepackage{amsmath,amssymb,amsfonts}
\usepackage{amsthm}
\usepackage{url}
\usepackage{graphicx}
\usepackage{wrapfig}
\usepackage{graphicx}
\usepackage{subcaption}
\usepackage{makecell}
\usepackage{caption}
\usepackage{placeins}
\usepackage{algorithm}
\usepackage{algpseudocode}
\usepackage{multirow}
\usepackage{wrapfig}
\usepackage{tabularx}
\usepackage{xspace}
\usepackage{url}
\usepackage{hyperref}       
\theoremstyle{plain}
\usepackage{array}

\newtheorem{proposition}{Proposition}[section]

\newtheorem{assumption}{Assumption}[section]
\usepackage{thmtools}
\usepackage{thm-restate}
\usepackage{etoc} 
\hypersetup{colorlinks=true, citecolor=[HTML]{3366CC}, linkcolor=[HTML]{3366CC}, urlcolor=[HTML]{3366CC}}

\definecolor{OursTint}{HTML}{F3F8FB}
\definecolor{SectionGray}{gray}{0.965}
\definecolor{LenGray}{gray}{0.45}
\definecolor{PosGreen}{HTML}{2E7D32}
\definecolor{NegRed}{HTML}{C0392B}
\definecolor{LenBlue}{HTML}{2C6F9E}
\definecolor{TieGray}{gray}{0.45}

\newcommand{\method}{\textsc{MPOPD}\xspace}

\newcommand{\acclen}[2]{#1{\scriptsize\textcolor{LenGray}{\,/\,#2}}}
\newcommand{\accplus}[1]{\textcolor{PosGreen}{\textbf{+#1}}}
\newcommand{\accminus}[1]{\textcolor{NegRed}{\textbf{\ensuremath{-}#1}}}

\newcommand{\lensave}[1]{{\scriptsize\textcolor{LenBlue}{\ensuremath{\downarrow}#1}}}
\newcommand{\lenmore}[1]{{\scriptsize\textcolor{NegRed}{\ensuremath{\uparrow}#1}}}
\newcommand{\impinline}[2]{#1{\scriptsize\textcolor{TieGray}{\,/\,}}#2}
\newcommand{\graybg}[1]{%
  \begingroup
  \setlength{\fboxsep}{1pt}%
  \colorbox{LenGray!15}{#1}%
  \endgroup
}
\newcolumntype{B}{>{\raggedright\arraybackslash}p{0.115\textwidth}}
\newcolumntype{M}{>{\centering\arraybackslash}p{0.105\textwidth}}
\newcolumntype{O}{>{\columncolor{OursTint}\centering\arraybackslash}p{0.105\textwidth}}
\newcolumntype{I}{>{\raggedleft\arraybackslash}p{0.125\textwidth}}

\newcolumntype{Y}{>{\raggedright\arraybackslash}X}

\definecolor{CaseGreen}{HTML}{DFF3E7}
\definecolor{CaseGreenText}{HTML}{14532D}
\definecolor{CaseRed}{HTML}{FEE2E2}
\definecolor{CaseRedText}{HTML}{7F1D1D}
\definecolor{CaseBlue}{HTML}{E8F0FE}
\definecolor{CaseBlueText}{HTML}{1D4ED8}
\definecolor{CaseAmber}{HTML}{FEF3C7}
\definecolor{CaseAmberText}{HTML}{92400E}
\definecolor{CaseGray}{HTML}{F6F7F9}

\newcommand{\casebadge}[3]{%
  \begingroup\setlength{\fboxsep}{1.5pt}%
  \colorbox{#1}{\strut\scriptsize\textcolor{#2}{\textbf{#3}}}%
  \endgroup%
}
\newcommand{\okbadge}{\casebadge{CaseGreen}{CaseGreenText}{Correct}}
\newcommand{\badbadge}[1]{\casebadge{CaseRed}{CaseRedText}{#1}}
\newcommand{\correctbadge}{\okbadge}
\newcommand{\wrongbadge}{\badbadge{Wrong}}
\newcommand{\lenbadge}[1]{\casebadge{CaseBlue}{CaseBlueText}{#1}}
\newcommand{\limitbadge}{%
  \begingroup\setlength{\fboxsep}{1.5pt}%
  \colorbox{CaseBlue}{\scriptsize\textcolor{CaseBlueText}{\textbf{\begin{tabular}{@{}c@{}}16.4k\\[-0.15em](max limit)\end{tabular}}}}%
  \endgroup%
}
\newcommand{\spanhl}[3]{%
  \begingroup\setlength{\fboxsep}{0.8pt}%
  \colorbox{#1}{\strut\textcolor{#2}{#3}}%
  \endgroup%
}
\newcommand{\goodhl}[1]{\spanhl{CaseGreen}{CaseGreenText}{#1}}
\newcommand{\patternhl}[1]{\spanhl{CaseAmber}{CaseAmberText}{#1}}
\newcommand{\wronghl}[1]{\spanhl{CaseRed}{CaseRedText}{#1}}
\newcommand{\rawline}[1]{{\ttfamily\scriptsize #1}}
\newcommand{\rawopenboxed}{\textbackslash boxed\textbraceleft{}}
\newcommand{\snippetcell}[1]{%
  \begin{minipage}[t]{\linewidth}\raggedright #1\end{minipage}%
}

\newcolumntype{Y}{>{\raggedright\arraybackslash}X}

\definecolor{CaseGreen}{HTML}{DFF3E7}
\definecolor{CaseGreenText}{HTML}{14532D}
\definecolor{CaseRed}{HTML}{FEE2E2}
\definecolor{CaseRedText}{HTML}{7F1D1D}
\definecolor{CaseBlue}{HTML}{E8F0FE}
\definecolor{CaseBlueText}{HTML}{1D4ED8}
\definecolor{CaseAmber}{HTML}{FEF3C7}
\definecolor{CaseAmberText}{HTML}{92400E}
\definecolor{CaseGray}{HTML}{F6F7F9}

\title{Proposing Better Rollouts for On-Policy Distillation}

\author{%
  Sirui Chen$^{1}$\thanks{Equal contribution.} \quad
  Yunzhe Qi$^{1}$\footnotemark[1] \quad
  Haobo Xu$^{1}$ \quad
  Jiaru Zou$^{2}$ \\
  \bfseries
  Mengting Ai$^{1}$ \quad
  Zihao Li$^{1}$ \quad
  Jingrui He$^{1}$ \\
  $^{1}$University of Illinois Urbana-Champaign \quad
  $^{2}$Stanford University \\
  \texttt{\{sirui6, yunzheq2\}@illinois.edu}
}

\begin{document}

\maketitle

\begin{abstract}
On-policy distillation (OPD) improves language-model post-training by letting the student learn from its own sampled trajectories under teacher supervision, aligning the training process with inference-time behavior. However, the rollout policy in OPD is usually treated as fixed: each round samples directly from the current student. We show that this choice can create a compounding teacher--student mismatch, because it requires the teacher to provide effective supervision on whichever trajectories the student samples, including ones less reachable by the teacher itself. On such trajectories, teacher supervision can become noisy or even misleading, leading to undesirable student patterns such as repetition. Moreover, since the updated student in turn defines the rollout policy of subsequent rounds, these patterns persist and potentially amplify over training. 
This failure mode suggests that improving OPD requires treating the rollout policy itself as an optimization target, rather than only modifying the distillation objective.
We propose \emph{Mixture Proposal On-Policy Distillation} (MPOPD), which derives a student-anchored, teacher-guided rollout distribution from a local trust-region objective based on relative teacher support. 
The resulting target is a geometric mixture of the student and teacher policies, and we introduce a proposal policy that approximates this mixture, enabling rollout generation in a single forward pass. 
Experiments on 16 math, knowledge, and medical reasoning benchmarks show consistent gains over OPD and recent variants. On math reasoning, our method improves average accuracy by 6.91 over OPD on Qwen3-4B-Base while reducing average generation length by 41\%, suggesting that rollout policy design provides an effective way to improve OPD.
Code is available at \url{https://github.com/csr16/MPOPD}.
\end{abstract}
\section{Introduction}
 
On-policy distillation~\cite{gu2024minillm,agarwal2024onpolicy,lu2025onpolicydistillation}
has emerged as an important post-training paradigm for large language models.
Unlike standard knowledge distillation~\cite{guha2025openthoughts,guo2025deepseek,zhao20251,hinton2015distilling}
or supervised fine-tuning~\cite{taori2023alpaca,bercovich2025llama},
which train the student on teacher- or expert-generated data, OPD lets the
student sample its own trajectories and then distills token-level teacher
supervision on those trajectories. This makes the training distribution better
match the student's inference-time behavior, and has led to strong results in
summarization, translation, and mathematical reasoning
tasks~\citep{xu2026tip,lu2025onpolicydistillation,wu2026lightning}.
Recent large-scale models including Qwen3.5, MiMo-V2, and DeepSeek-V4 have
therefore adopted on-policy distillation as a core component of their
post-training pipelines~\citep{team2026qwen3,coreteam2025mimounlockingreasoningpotential,deepseekai2026deepseekv4}.

The fine-grained, on-policy supervision is what makes OPD powerful, but it also exposes OPD to a potential teacher--student support mismatch. 
Since trajectories are sampled from the student, the teacher is expected to provide token-level supervision not only on responses it would naturally generate, but also on student trajectories that often have low probability under the teacher. On such low-reachability trajectories, the teacher's conditional token distribution can become noisy or even misleading.
This mismatch can be self-reinforcing across distillation rounds. If a
degenerate pattern receives insufficient penalty in one round, it can remain in
the student's high-probability region, be sampled again in later rounds, and
receive similarly weak or noisy teacher supervision \citep{xuspeculative}. Recent studies
\citep{li2026rethinking} have observed related failure modes in
OPD-trained students, including repetition~\citep{fu2026revisiting,gu2024minillm}
and entropy collapse~\citep{jin2026entropy,ko2026scaling}. These findings
suggest that the effectiveness of OPD depends not only on using on-policy
samples, but also on whether the teacher can provide reliable supervision on
the particular regions of the student distribution visited during training.

\begin{wrapfigure}{r}{0.41\linewidth}
    \vspace{-1.5em}
    \centering
    \includegraphics[width=\linewidth]{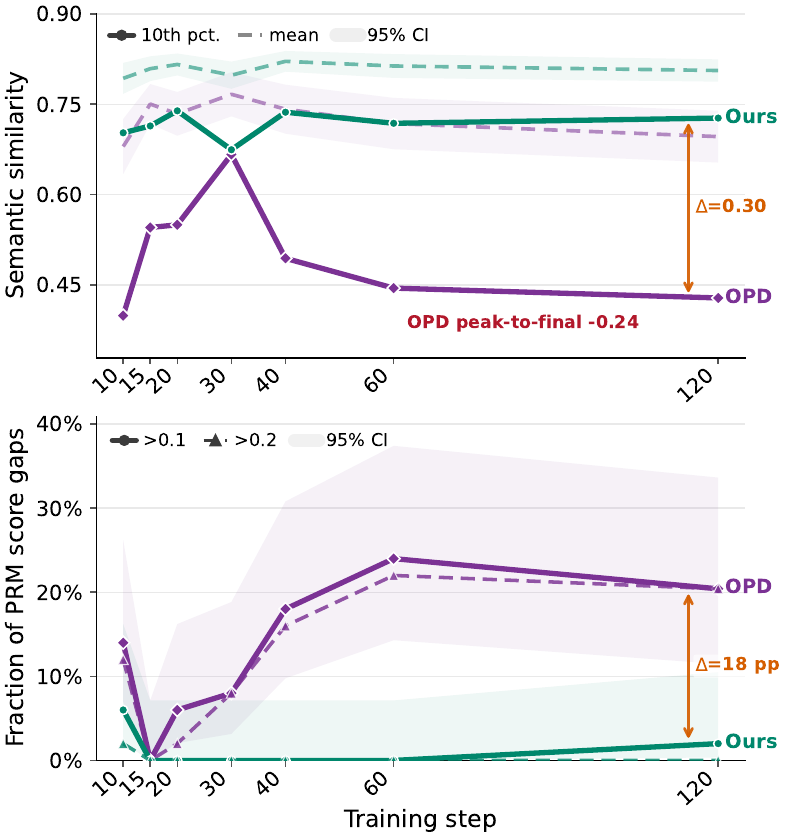}
    \captionsetup{font=footnotesize}
    \vspace{-0.2in}
    \caption{\textbf{Teacher--student mismatch compounds under OPD}.
    \textit{Upper}: 10th-percentile similarity between the two completions drops by $0.24$ from its peak, while our method remains stable. \textit{Lower}: fraction of prompts on which the former scores $>0.10$ (solid) or $>0.20$ (dashed) below the latter in PRM score; the former reaches ${\sim}24\%$ by the end of training, while our method keeps it near zero.}
    \label{fig:degradation}
    \vspace{-1.0em}
\end{wrapfigure}

We verify this \emph{compounding effect} empirically. At selected training steps, we use the teacher to
complete a student trajectory and a teacher trajectory from
the same prompt.\footnote{We compare the two completions by (i) semantic
similarity and (ii) the fraction of prompts where the student-trajectory
completion gets a lower process reward model (PRM) score
than the teacher-trajectory one, at gap thresholds
$0.1$ and $0.2$.} Fig.~\ref{fig:degradation} shows two trends
under OPD. First, the similarity between the two teacher completions drifts
downward over training, with the 10th percentile dropping more sharply, indicating that a small but growing tail of student trajectories is rarely
reachable under the teacher. 
Second, this drift also degrades completion quality. The fraction of prompts
where the teacher's completion from a student trajectory receives a lower
PRM reward than its completion from its own trajectory grows steadily under
both gap thresholds.
Both trends are
consistent with the mechanism above: once the student places probability on
trajectories the teacher would rarely reach itself, subsequent rounds
reinforce them, and the teacher--student mismatch persists across iterations.

To break this loop, we propose to optimize the rollout policy directly, shifting probability toward trajectories the teacher is more likely to reach while remaining close enough to the current student to maintain the on-policy nature. Concretely, we introduce a reweighting function based on relative teacher support, which upweights trajectories the teacher supports more strongly than the current student. To stay close to on-policy behavior, we constrain this reweighting to a local trust region around the current student. Solving this problem yields a geometric mixture of the student and teacher policies, with a single scalar interpolating between on-policy and teacher-guided sampling.
However, computing this mixture online requires running both the teacher and student at every decoding step, which is costly at scale. We therefore train a proposal policy to approximate it, so that rollouts need only a single forward pass. This gives our \emph{Mixture Proposal On-Policy Distillation} (\method), where the rollout policy is first derived as a geometric mixture of the student and teacher and then approximated by a proposal policy for efficient generation.

Prior efforts to improve OPD have mainly modified the distillation
objective~\citep{jang2026stable,jin2026entropy,yang2026learning,ko2026scaling},
or adjusted the rollout policy via heuristic sampling~\citep{xuspeculative,gu2024minillm}.  In contrast, we treat the rollout policy itself as an explicit optimization target, derived from a local trust-region principle and approximated by a learnable proposal policy for scalable generation. Fig.~\ref{fig:degradation} shows that our method avoids the compounding mismatch, with 10th-percentile similarity 0.30 higher and PRM score gap fraction 18 points lower than OPD. We further evaluate \method on 16 math, knowledge, and medical reasoning
benchmarks. Across these tasks, \method consistently outperforms OPD and recent variants. On math,
it improves average accuracy over OPD by +1.46 on
Qwen3-1.7B-Base and +6.91 on Qwen3-4B-Base, while reducing
generation length by 41\% at the 4B scale.

\begin{figure*}[t]
    \centering
    \vspace{-0.3in}\includegraphics[width=0.88\linewidth]{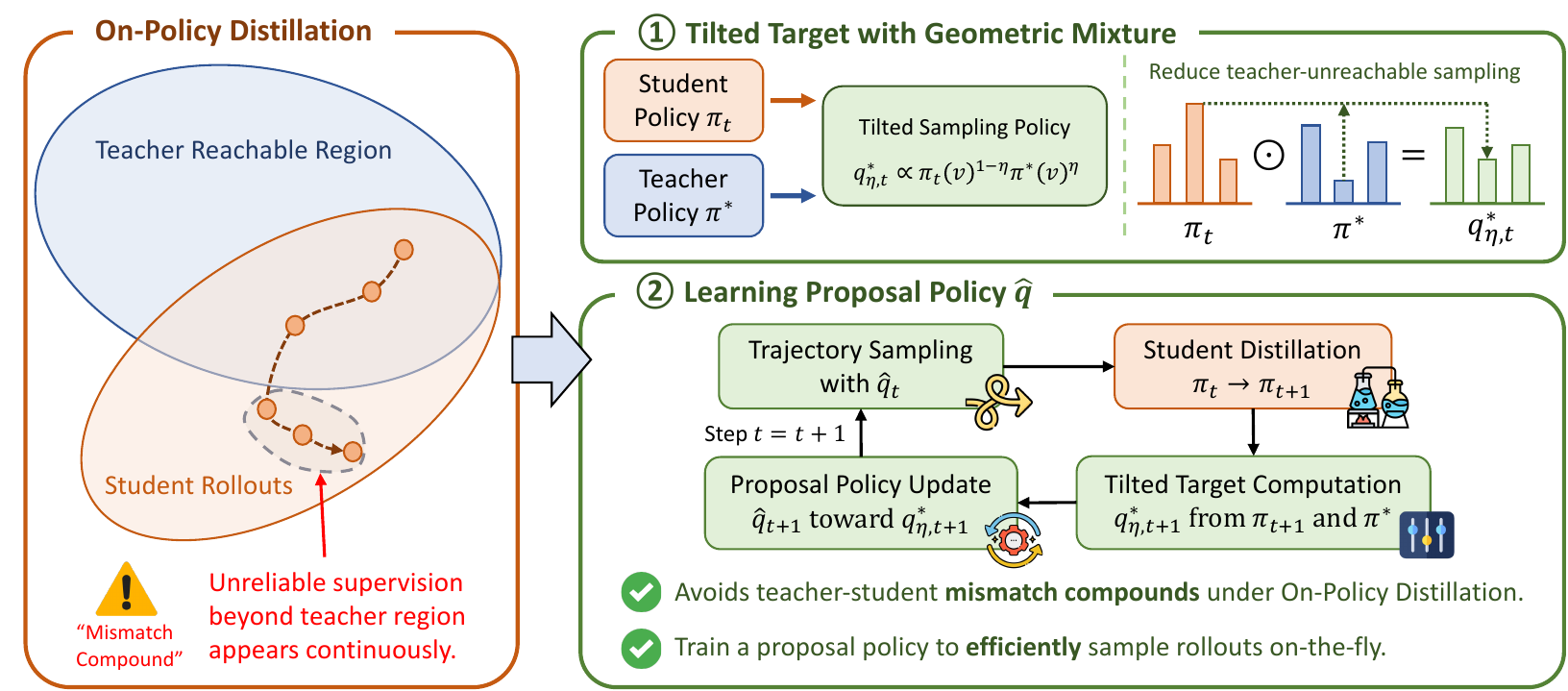}
    \caption{\textbf{Overview of \method.}
    \textbf{Left:} Standard on-policy distillation suffers from compounding teacher--student mismatch, where student rollouts continuously drift beyond the teacher-reachable region and receive unreliable supervision.
    \textbf{Right:} \method derives a tilted target as a geometric mixture of the student and teacher, and trains a proposal policy to approximate it.}
\label{fig:main}
\end{figure*}
\section{Setup: The Rollout Policy in On-Policy Distillation}
\label{sec:setup}

Let \(\mathcal{D}\) denote the prompt distribution over input problems \(x\),
and let \(\mathcal{V}\) denote the vocabulary. A response trajectory is a token sequence $\tau = (y_1, \ldots, y_j, \ldots, y_{|\tau|})$ with $y_j \in \mathcal{V}$, and a policy $\pi$ defines the autoregressive distribution $\pi(\tau \mid x) = \prod_{j=1}^{|\tau|} \pi(y_j \mid x, y_{<j})$. We denote the teacher policy by $\pi^\star$ and the student policy at the $t$-th distillation round by $\pi_t$. 
At each round $t$, OPD samples response trajectories from a rollout policy $q_t$, queries the teacher on the sampled responses, and updates the student against teacher supervision through a token-level objective:
\begin{equation}
\mathcal{L}(\pi; q_t) = \mathbb{E}_{x \sim \mathcal{D},\, \tau \sim q_t(\cdot \mid x)}\!\left[\sum_{j=1}^{|\tau|} \mathrm{KL}\!\left(\pi(\cdot \mid x, y_{<j}) \,\big\|\, \pi^\star(\cdot \mid x, y_{<j})\right)\right].
\label{eq:opd-objective}
\end{equation}
In standard OPD, the rollout policy is taken to be the current student, $q_t = \pi_t$, so the student learns from its own behavior.
While this choice aligns training trajectories with the student's inference-time behavior, it also exposes the teacher to low-reachability student trajectories, producing the compounding mismatch over training rounds (Fig.~\ref{fig:degradation}). We therefore treat the rollout policy as a \textbf{design choice} and seek a $q_t$ that (i) upweights tokens that are more reachable under the teacher, and (ii) deviates from $\pi_t$ only locally, so the modification acts as a correction rather than a replacement.

A natural way to implement (i) and (ii) structurally is to write the rollout policy as a multiplicative reweighting of the student policy. Specifically, for a prefix $(x, y_{<j})$ and a candidate next token $v\in\mathcal{V}$,
\begin{equation}
q_t(v \mid x, y_{<j}) \;\propto\; \pi_t(v \mid x, y_{<j})\, h(v, x, y_{<j}), 
\label{eq:reweighted-proposal}
\end{equation}
where $h$ is a nonnegative reweighting function. The factor $\pi_t$ anchors $q_t$ to the current student, while $h$ redistributes probability mass across candidate tokens at each generation step. A good $h$ should perform this redistribution so that the rollout shifts toward trajectories better aligned with the teacher's own behavior. Building on this requirement, we design a local score that captures which tokens should be upweighted, and show that the resulting $h$ takes an exponential-tilt form, yielding $q_t$ as a geometric mixture of $\pi_t$ and $\pi^\star$ (\S\ref{sec:reweighting}). To use this rollout policy efficiently on the fly, we then approximate it with a learnable proposal policy that samples rollouts in a single forward pass (\S\ref{sec:proposal}).
\section{Mixture Proposal for On-Policy Distillation}
\label{sec:method}

\subsection{Reweighting by Relative Teacher Support}
\label{sec:reweighting}
The reweighting framework in \S\ref{sec:setup} requires a metric that determines which tokens should receive more probability. Here, we choose the \textit{relative teacher support}: for a token $v$ at prefix $(x,y_{<j})$, define
\begin{equation*}
s_t(v,x,y_{<j})
=
\log \pi^\star(v \mid x,y_{<j})
-
\log \pi_t(v \mid x,y_{<j}).
\label{eq:relative-teacher-support}
\end{equation*}
A high $s_t$ means that the teacher assigns larger probability to $v$ than the student does. Upweighting such tokens follows two intuitions. First, high teacher probability suggests that the teacher can provide more effective supervision, avoiding noisy supervision on tokens the teacher itself rarely generates. Second, low student probability suggests that the student has not yet fully learned these tokens; if the student already assigns high probability, further emphasizing it provides less benefit. Reweighting by $s_t$ therefore prioritizes tokens that are both well supported by the teacher and not yet captured by the student.

Having specified the reweighting signal, we next control how strongly it shifts the rollout distribution away from the student. We choose $q_t$ by maximizing the expected relative teacher support within a local trust region around the student in KL divergence:
\begin{equation}\label{eq:local-trust-region}
    \max_{q(\cdot\mid x,y_{<j})}
    \;
    \mathbb{E}_{v\sim q(\cdot\mid x,y_{<j})}
    \!\left[
        s_t(v,x,y_{<j})
    \right]
    \quad
    \mathrm{s.t.}
    \quad
    \mathrm{KL}
    \!\left(
        q(\cdot\mid x,y_{<j})
        \,\big\|\,
        \pi_t(\cdot\mid x,y_{<j})
    \right)
    \leq
    \delta .
\end{equation}
The objective encourages tokens with larger relative teacher support, while the KL constraint prevents the rollout distribution from moving too far from the current student policy. Solving this constrained problem yields a student-anchored exponential
reweighting of $\pi_t$:
\begin{restatable}[Geometric mixture of the tilted target]{proposition}{geomixture}
\label{prop:geometric-mixture}
For a fixed prefix $(x,y_{<j})$, there exists a multiplier $\eta\ge 0$
such that the solution to Eq.~\eqref{eq:local-trust-region} takes the form
\begin{equation}\label{eq:local-exponential-tilt}
    q_{\eta,t}^*(v\mid x,y_{<j})
    =
    \frac{
        \pi_t(v\mid x,y_{<j})
        \exp\left(\eta s_t(v,x,y_{<j})\right)
    }{
        Z(x,y_{<j})
    },
\end{equation}
where $\eta$ is the inverse Lagrange multiplier associated with the KL constraint and
$Z(x,y_{<j})=\sum_{v'\in\mathcal{V}}\pi_t(v'\mid x,y_{<j})\exp\left(\eta s_t(v',x,y_{<j})\right)$.
Equivalently, substituting the definition of $s_t$ gives
\begin{equation}\label{eq:geometric-mixture}
    q_{\eta,t}^*(v\mid x,y_{<j})
    =
    \frac{
        \pi_t(v\mid x,y_{<j})^{1-\eta}
        \pi^\star(v\mid x,y_{<j})^{\eta}
    }{
        \sum_{v'\in\mathcal{V}}
        \pi_t(v'\mid x,y_{<j})^{1-\eta}
        \pi^\star(v'\mid x,y_{<j})^{\eta}
    } .
\end{equation}
\end{restatable}

The proof is deferred to Appendix~\ref{app:proofs_geometric-mixture}. The solution is a
geometric mixture of the student and teacher policy, governed by a
prefix-dependent scalar $\eta$. 
In this case, the reweighting function in Eq.~\eqref{eq:reweighted-proposal} becomes an \emph{exponential tilt},
$h(v,x,y_{<j})=\exp(\eta s_t(v,x,y_{<j}))$. Here, $s_t$ determines which tokens are upweighted, while $\eta$ controls the strength of this reweighting.
A further discussion, including a comparison with arithmetic mixture, is deferred to Appendix~\ref{app:geometric-arithmetic-interpolation}. In practice, we can restrict $\eta\in[0,1]$ and treat it as a tunable parameter, which yields a continuous
interpolation between the two endpoints: $\eta=0$ recovers $\pi_t$, while
$\eta=1$ recovers $\pi^\star$. Intermediate values balance student anchoring and teacher guidance, with larger $\eta$ placing more mass on tokens with higher relative teacher support.

The improvement in expected relative teacher support from tilting depends on the spread of $s_t$ across the vocabulary. To make this precise, for a fixed prefix $(x,y_{<j})$, let $S_j(q)=\mathbb{E}_{v\sim q}[s_t(v,x,y_{<j})]$ denote the expected relative teacher support under a distribution $q$, and let $\sigma_j(x,y_{<j})=\sqrt{\mathrm{Var}_{v\sim \pi_t}[s_t(v,x,y_{<j})]}$ denote its standard deviation under the student. The following corollary quantifies this improvement under the trust-region constraint.
\begin{restatable}[Local improvement from tilting]{corollary}{localimprovement}
\label{cor:local-improvement}
For a fixed prefix $(x,y_{<j})$, consider the local trust-region problem
in Eq.~\eqref{eq:local-trust-region}. For
sufficiently small \(\delta\),
\begin{equation}
    S_j(q^*_\eta) - S_j(\pi_t)
    =
    \sigma_j(x,y_{<j})\sqrt{2\delta}
    +
    o(\sqrt{\delta}) .
\end{equation}
\end{restatable}

The proof is deferred to Appendix~\ref{app:proofs_local-improvement}. The improvement is
governed by two quantities. The local standard deviation $\sigma_j$
formalizes the spread above: when $\sigma_j$ is small the student already agrees with the teacher up to a constant offset and
tilting has little to redistribute, while a larger $\sigma_j$ indicates
that high- and low-score tokens coexist in the student policy and
tilting can reallocate mass between them. The trust-region radius
$\delta$ controls the allowed deviation from the student, with
$\sqrt{\delta}$ scaling reflecting diminishing returns from larger
local deviations. Appendix~\ref{app:structured-teacher-target-noise} further shows that, under a structured teacher noise model, tilted rollouts reduce exposure to noisy teacher supervision provided a covariance condition holds.

\subsection{Learning a Proposal Rollout Policy}
\label{sec:proposal}
The tilted rollout policy in Eq.~\eqref{eq:geometric-mixture} is expensive to
sample from directly: each decoding step would require running both the student
and teacher, combining their logits, and sampling from the resulting mixture.
This online construction requires synchronizing two forward passes throughout
generation, making it difficult to implement efficiently at scale.

To avoid this online cost, we introduce a proposal policy
$\hat q_t$ that approximates the tilted target. At round $t$,
trajectories are sampled from $\hat q_t$, while the student update is
still defined against $\pi_t$ and $\pi^\star$. The student is updated by
\begin{equation}\label{eq:student-update}
    \pi_{t+1}
    =
    \arg\min_{\pi}
    \mathbb{E}_{x\sim\mathcal{D},\,\tau\sim\hat q_t(\cdot\mid x)}
    \left[
        \sum_{j=1}^{|\tau|}
        w_{t,j-1}(\tau)
        \,
        \mathrm{KL}
        \left(
            \pi(\cdot\mid x,y_{<j})
            \,\|\,
            \pi^\star(\cdot\mid x,y_{<j})
        \right)
    \right],
\end{equation}
where $w_{t,j}(\tau)$ corrects for the gap between the proposal policy $\hat q_t$ and the student $\pi_t$. Following prior work~\citep{gu2024minillm,levine2020offline}, we use the clipped one-step ratio
\begin{equation}\label{eq:clipped_IS-_ratio}
    w_{t,j}(\tau)=\mathrm{clip}\!\left( \frac{\pi_t(v\mid x,y_{<j}) }{ \hat q_t(v\mid x,y_{<j})} ,\,1/C,\,C\right)
\end{equation}
with $w_{t,0}(\tau)=1$ and clipping threshold $C>1$.
To learn the proposal policy for the next round, we form the tilted target
at the updated student,
\begin{equation}\label{eq:target}
    q_{\eta,t+1}^*(v\mid x,y_{<j})
    =
    \frac{
        \pi_{t+1}(v\mid x,y_{<j})^{1-\eta}
        \pi^\star(v\mid x,y_{<j})^{\eta}
    }{
        \sum_{v'\in\mathcal{V}}
        \pi_{t+1}(v'\mid x,y_{<j})^{1-\eta}
        \pi^\star(v'\mid x,y_{<j})^{\eta}
    },
\end{equation}
which serves as the target for updating $\hat q_{t+1}$:
\begin{equation}\label{eq:qhat-update}
    \hat q_{t+1}
    =
    \arg\min_{\hat q}
    \mathbb{E}_{x\sim\mathcal{D},\,\tau\sim\hat q_t(\cdot\mid x)}
    \left[
        \sum_{j=1}^{|\tau|}
        \mathrm{KL}
        \left(
            \hat q(\cdot\mid x,y_{<j})
            \,\|\,
            q_{\eta,t+1}^*(\cdot\mid x,y_{<j})
        \right)
    \right].
\end{equation}

The procedure alternates between two updates: the student update uses trajectories sampled from the proposal $\hat q_t$, while the proposal update moves $\hat q_t$ toward the next tilted target $q_{\eta,t+1}^*$. As a result, the teacher is queried only during training to define the learning target, while rollout generation uses only $\hat q_t$. The full procedure (Algorithm~\ref{alg:amortized-tilted-sampling}) provides an efficient on-the-fly approximation to the geometric mixture rollout while keeping student updates aligned with the OPD objective. Appendix~\ref{app:additional_proposal} evaluates the quality of this approximation. Further discussion is in Appendix~\ref{app:discussion}.

\begin{algorithm}[t]
\caption{Mixture Proposal for On-Policy Distillation}
\label{alg:amortized-tilted-sampling}
\begin{algorithmic}[1]
\Require initial student $\pi_0$, teacher $\pi^\star$, training prompt distribution  $\mathcal{D}$, rounds $R$, tilt strength $\eta$.
\State Initialize proposal policy $\hat q_0 \leftarrow \pi_0$
\For{$t=0,1,\ldots,R-1$}
    \State Sample a batch of prompts from $\mathcal{D}$
    \State Generate trajectories $\tau\sim \hat q_t(\cdot\mid x)$ for each prompt $x$
    \State Update the student $\pi_{t+1}$ by the distillation objective with importance weights $w_{t,j}(\tau)$ in Eq.~\eqref{eq:clipped_IS-_ratio}: 
    \[
        \pi_{t+1}
        =
        \arg\min_{\pi}
        \mathbb{E}_{x\sim\mathcal{D},\,\tau\sim\hat q_t(\cdot\mid x)}
        \left[
            \sum_{j=1}^{|\tau|}
            w_{t,j-1}(\tau)
            \,
            \mathrm{KL}
            \left(
                \pi(\cdot\mid x,y_{<j})
                \,\|\,
                \pi^\star(\cdot\mid x,y_{<j})
            \right)
        \right].
    \]
    \State Define the tilted target $q_{\eta,t+1}^*$ from $\pi_{t+1}$ and $\pi^\star$ using Eq.~\eqref{eq:target}
    \State Update $\hat q_{t+1}$ toward target $q_{\eta,t+1}^*$ based on  
    \[
       \hat q_{t+1}
        =
        \arg\min_{\hat q}
        \mathbb{E}_{x\sim\mathcal{D},\,\tau\sim\hat q_t(\cdot\mid x)}
        \left[
            \sum_{j=1}^{|\tau|}
            \mathrm{KL}
            \left(
                \hat q(\cdot\mid x,y_{<j})
                \,\|\,
                q_{\eta,t+1}^*(\cdot\mid x,y_{<j})
            \right)
        \right].
    \]
\EndFor
\State \Return final student $\pi_R$
\end{algorithmic}
\end{algorithm}

\section{Experiments}
\label{sec:experiment}
\subsection{Experimental Setup}

\paragraph{Models and datasets.}
We evaluate two student scales in the Qwen3 model family~\citep{yang2025qwen3}, Qwen3-1.7B-Base and Qwen3-4B-Base, distilled from Qwen3-4B-Instruct-2507 as the teacher. For the main math reasoning experiments, we train on
DAPO-Math-17K~\citep{yu2025dapo} and evaluate on AIME-24, AIME-25, AIME-26,
AMC-23, HMMT 2025 (Feb \& Nov)~\citep{balunovic2025matharena}, Minerva
Math~\citep{lewkowycz2022solving}, and
OlympiadBench~\citep{he2024olympiadbench}. We report avg@32 for all benchmarks except OlympiadBench, where we report avg@8 due to its large size. For additional evaluation, we distill Qwen3-4B-Base on two domains: knowledge QA
and medical reasoning. In knowledge QA, the model is trained on
ARPO-10k~\citep{dong2025agentic} and evaluated on
GPQA-Diamond~\citep{rein2023gpqa}, ARC-Challenge~\citep{clark2018think}, and
MMLU-Pro~\citep{wang2024mmlu}, with avg@8 reported for all three benchmarks.
In medical reasoning, we use
M23K~\citep{huang2025m1} for training and evaluate on
PubMed~\citep{jin2019pubmedqa}, MedQA~\citep{jin2021disease}, Lancet \&
NEJM~\citep{huang2025m1}, and MedX~\citep{zuo2025medxpertqa}. Following
\cite{huang2025m1}, we report avg@1. Additional details are provided in Appendix~\ref{app:details_evaluation}.

\paragraph{Baselines.} In addition to standard OPD, we compare against three recent variants: ExOPD~\cite{yang2026learning}, which applies an extrapolative supervision relative to the student base model; EOPD~\cite{jin2026entropy}, which augments OPD with an entropy-aware forward-KL term at high-entropy teacher positions; and VETO~\cite{jang2026stable}, which uses a geometric mixture of the student and teacher as the distillation target. 

\paragraph{Implementation details.}
All methods are trained using the \texttt{verl}~\cite{sheng2024hybridflow} framework
with vLLM~\cite{kwon2023efficient} as the rollout generation backend.
We use one rollout per prompt, rollout temperature $0.8$, and batch size $256$ for math reasoning and $128$ for the
additional knowledge QA and medical reasoning experiments.
Local-support methods use teacher top-$16$ support.
The proposal policy $\hat{q}_t$ is a full model
initialized from the current student at each round.
All experiments are conducted on GH200 GPUs.
Additional implementation details are provided in Appendix~\ref{app:details_implementation}.

\subsection{Main Results}

\begin{table*}[t]
\centering
\footnotesize
\setlength{\tabcolsep}{3.5pt}
\renewcommand{\arraystretch}{1.08}
\vspace{-0.1in}
\caption{Evaluation results on \textbf{math benchmarks}. Each cell reports accuracy (\%) and \graybg{average} \graybg{generation length}. \textbf{Imp.} compares \method with the best baseline in each row, showing the accuracy gain in percentage points and the relative length change. Best and second-best accuracies within each model block are bolded and underlined, respectively.}
\begin{tabular}{@{}lccccc>{\columncolor{OursTint}}cc@{}}
\toprule
\rowcolor{SectionGray}
\multicolumn{8}{c}{\textbf{Qwen3-1.7B-Base}} \\
\midrule
\textbf{Benchmark} 

& \textbf{Base} & \textbf{OPD} & \textbf{ExOPD} & \textbf{EOPD} & \textbf{VETO} 
& \cellcolor{OursTint}\textbf{\method} & \textbf{Imp.} \\
\midrule
AMC23
& \acclen{13.05}{6.96k}
& \acclen{39.22}{16.24k}
& \acclen{40.00}{16.37k}
& \acclen{\underline{41.02}}{16.16k}
& \acclen{39.14}{10.30k}
& \acclen{\textbf{41.09}}{9.78k}
& \impinline{\accplus{0.07}}{\lensave{39.5\%}} \\

AIME24
& \acclen{1.67}{8.00k}
& \acclen{11.46}{16.38k}
& \acclen{\textbf{12.92}}{16.38k}
& \acclen{11.56}{16.35k}
& \acclen{10.62}{14.90k}
& \acclen{\underline{12.50}}{14.65k}
& \impinline{\accminus{0.42}}{\lensave{10.6\%}} \\

AIME25
& \acclen{1.15}{7.66k}
& \acclen{7.19}{16.36k}
& \acclen{7.29}{16.38k}
& \acclen{\underline{8.65}}{16.29k}
& \acclen{7.40}{15.06k}
& \acclen{\textbf{8.75}}{14.79k}
& \impinline{\accplus{0.10}}{\lensave{9.2\%}} \\

AIME26
& \acclen{0.73}{7.73k}
& \acclen{7.71}{16.33k}
& \acclen{7.71}{16.37k}
& \acclen{7.81}{16.31k}
& \acclen{\underline{7.92}}{15.00k}
& \acclen{\textbf{9.06}}{14.85k}
& \impinline{\accplus{1.14}}{\lensave{1.0\%}} \\

H25F
& \acclen{0.21}{7.82k}
& \acclen{\underline{3.12}}{16.32k}
& \acclen{2.92}{16.37k}
& \acclen{\underline{3.12}}{16.32k}
& \acclen{2.81}{15.50k}
& \acclen{\textbf{3.65}}{15.48k}
& \impinline{\accplus{0.53}}{\lensave{5.2\%}} \\

H25N
& \acclen{0.63}{6.78k}
& \acclen{\underline{1.46}}{16.34k}
& \acclen{1.35}{16.37k}
& \acclen{1.25}{16.34k}
& \acclen{1.35}{15.88k}
& \acclen{\textbf{2.29}}{15.74k}
& \impinline{\accplus{0.83}}{\lensave{3.7\%}} \\

Minerva
& \acclen{8.96}{6.90k}
& \acclen{24.35}{15.81k}
& \acclen{25.57}{15.14k}
& \acclen{25.88}{15.74k}
& \acclen{\underline{26.86}}{8.57k}
& \acclen{\textbf{28.25}}{7.02k}
& \impinline{\accplus{1.39}}{\lensave{18.0\%}} \\

Olymp.
& \acclen{12.54}{6.65k}
& \acclen{33.85}{16.24k}
& \acclen{33.72}{16.36k}
& \acclen{\underline{34.20}}{16.20k}
& \acclen{32.94}{11.38k}
& \acclen{\textbf{34.51}}{10.77k}
& \impinline{\accplus{0.31}}{\lensave{33.5\%}} \\

\midrule
\textbf{Avg.}
& \acclen{4.87}{7.31k}
& \acclen{16.05}{16.25k}
& \acclen{16.43}{16.22k}
& \acclen{\underline{16.69}}{16.21k}
& \acclen{16.13}{13.32k}
& \acclen{\textbf{17.51}}{12.89k}
& \impinline{\accplus{0.82}}{\lensave{20.5\%}} \\

\midrule
\rowcolor{SectionGray}
\multicolumn{8}{c}{\textbf{Qwen3-4B-Base}} \\
\midrule
\textbf{Benchmark} 
& \textbf{Base} & \textbf{OPD} & \textbf{ExOPD} & \textbf{EOPD} & \textbf{VETO} 
& \cellcolor{OursTint}\textbf{\method} & \textbf{Imp.} \\
\midrule
AMC23
& \acclen{13.44}{1.87k}
& \acclen{56.33}{16.26k}
& \acclen{56.80}{16.34k}
& \acclen{61.25}{16.32k}
& \acclen{\underline{63.12}}{6.51k}
& \acclen{\textbf{69.61}}{6.08k}
& \impinline{\accplus{6.49}}{\lensave{6.6\%}} \\

AIME24
& \acclen{2.50}{1.72k}
& \acclen{23.02}{16.37k}
& \acclen{22.50}{16.34k}
& \acclen{22.19}{16.31k}
& \acclen{\underline{23.85}}{12.47k}
& \acclen{\textbf{28.65}}{12.32k}
& \impinline{\accplus{4.80}}{\lensave{1.2\%}} \\

AIME25
& \acclen{2.08}{1.37k}
& \acclen{23.44}{16.37k}
& \acclen{22.29}{16.38k}
& \acclen{23.44}{16.30k}
& \acclen{\underline{24.69}}{11.63k}
& \acclen{\textbf{28.96}}{11.42k}
& \impinline{\accplus{4.27}}{\lensave{1.8\%}} \\

AIME26
& \acclen{0.83}{1.77k}
& \acclen{19.69}{16.33k}
& \acclen{\underline{21.15}}{16.33k}
& \acclen{20.73}{16.31k}
& \acclen{20.83}{11.27k}
& \acclen{\textbf{25.00}}{11.27k}
& \impinline{\accplus{3.85}}{\lensave{31.0\%}} \\

H25F
& \acclen{0.42}{1.83k}
& \acclen{\underline{13.02}}{16.35k}
& \acclen{12.19}{16.34k}
& \acclen{12.71}{16.35k}
& \acclen{11.56}{13.06k}
& \acclen{\textbf{15.83}}{13.09k}
& \impinline{\accplus{2.81}}{\lensave{20.0\%}} \\

H25N
& \acclen{2.50}{1.92k}
& \acclen{12.19}{16.37k}
& \acclen{\underline{12.81}}{16.38k}
& \acclen{12.19}{16.35k}
& \acclen{10.73}{13.70k}
& \acclen{\textbf{15.42}}{13.32k}
& \impinline{\accplus{2.61}}{\lensave{18.7\%}} \\

Minerva
& \acclen{9.88}{2.26k}
& \acclen{32.74}{16.28k}
& \acclen{33.80}{16.27k}
& \acclen{34.79}{16.27k}
& \acclen{\underline{40.26}}{2.52k}
& \acclen{\textbf{42.65}}{2.43k}
& \impinline{\accplus{2.39}}{\lensave{3.4\%}} \\

Olymp.
& \acclen{11.42}{1.81k}
& \acclen{47.24}{16.28k}
& \acclen{47.39}{16.33k}
& \acclen{47.42}{16.30k}
& \acclen{\underline{53.41}}{7.14k}
& \acclen{\textbf{56.82}}{7.13k}
& \impinline{\accplus{3.41}}{\lensave{0.1\%}} \\

\midrule
\textbf{Avg.}
& \acclen{5.38}{1.82k}
& \acclen{28.46}{16.33k}
& \acclen{28.62}{16.34k}
& \acclen{29.34}{16.31k}
& \acclen{\underline{31.06}}{9.79k}
& \acclen{\textbf{35.37}}{9.63k}
& \impinline{\accplus{4.31}}{\lensave{1.6\%}} \\

\bottomrule
\end{tabular}
\vspace{-0.22in}
\label{tab:main}
\end{table*}
\paragraph{Main results on math reasoning.}
Table~\ref{tab:main} reports accuracy and average generation length across eight
math benchmarks at two student scales. On Qwen3-4B-Base, our method achieves the
best accuracy on all eight benchmarks, improving over the strongest baseline by
$+4.31$ on average. On Qwen3-1.7B-Base, our method is best on 7 out of 8 benchmarks, with an average gain of $+0.82$. Appendix~\ref{app:additional_scaling} further scales to a larger teacher and student. Overall, these results show that optimizing the rollout policy itself can further improve OPD beyond modifying the distillation objective or target distribution. The comparison with VETO is particularly informative: VETO also uses a geometric mixture of the student and teacher, but applies it as the
\emph{distillation target} rather than the \emph{rollout policy}. The gap between VETO and our method thus highlights the benefit of reweighting the rollout distribution rather than the distillation target.

Table~\ref{tab:main} also shows that our method reduces generation length.
While some baselines often approach the max token limit, our method
generates substantially shorter responses while improving accuracy. On Qwen3-4B-Base, our method reduces the average generation length by $41\%$ relative to OPD while
improving average accuracy by $+6.91$. This length gap mainly comes from the baselines producing degenerate responses that fail to terminate, rather than from our method truncating valid reasoning. Appendix~\ref{app:additional_decoding} further confirms this with a finer-grained analysis.
Fig.~\ref{fig:case-tritium} provides a representative example: our method produces a concise correct solution in 319 tokens, whereas OPD reaches the 16.4k limit with repeated final-answer patterns and EOPD reaches the same limit with spurious problem-rejection behavior, ultimately answering incorrectly. Both failure modes (benign repetition and active divergence) are consistent with the mechanism framed in \S\ref{sec:setup}, where rollouts drifting beyond the teacher's reachable region receive unreliable supervision.
Together, these results suggest that the tilted
proposal policy helps avoid the degenerate patterns observed in
standard OPD~\cite{fu2026revisiting,gu2024minillm}, rather than merely trading
accuracy for shorter outputs.

\begin{table*}[t]
\centering
\footnotesize
\setlength{\tabcolsep}{3.5pt}
\renewcommand{\arraystretch}{1.08}
\caption{Evaluation results on \textbf{knowledge QA and medical reasoning} benchmarks on Qwen3-4B-Base. Each cell reports accuracy (\%) and \graybg{average generation length}. \textbf{Imp.} compares \method with the best baseline in each row, showing the accuracy gain in percentage points and the relative length change. Best and second-best accuracies within each section are bolded and underlined, respectively.}
\begin{tabular}{@{}lccccc>{\columncolor{OursTint}}cc@{}}
\toprule
\rowcolor{SectionGray}
\multicolumn{8}{c}{\textbf{Knowledge Question Answering}} \\
\midrule
\textbf{Benchmark} 
& \textbf{Base} & \textbf{OPD} & \textbf{ExOPD} & \textbf{EOPD} & \textbf{VETO} 
& \cellcolor{OursTint}\textbf{\method} & \textbf{Imp.} \\
\midrule
GPQA-D
& \acclen{7.32}{8.22k}
& \acclen{38.51}{16.37k}
& \acclen{40.03}{16.38k}
& \acclen{\underline{41.60}}{16.35k}
& \acclen{40.47}{5.78k}
& \acclen{\textbf{42.17}}{6.34k}
& \impinline{\accplus{0.57}}{\lensave{61.3\%}} \\

ARC-C
& \acclen{24.82}{3.92k}
& \acclen{88.52}{8.18k}
& \acclen{85.38}{8.19k}
& \acclen{87.95}{8.17k}
& \acclen{\underline{91.45}}{0.33k}
& \acclen{\textbf{92.37}}{0.32k}
& \impinline{\accplus{0.92}}{\lensave{5.1\%}} \\

MMLU-P
& \acclen{9.30}{5.58k}
& \acclen{62.45}{8.19k}
& \acclen{63.75}{8.20k}
& \acclen{62.82}{8.15k}
& \acclen{\underline{64.10}}{1.42k}
& \acclen{\textbf{65.08}}{1.31k}
& \impinline{\accplus{0.98}}{\lensave{7.2\%}} \\

\midrule
\textbf{Avg.}
& \acclen{13.81}{5.91k}
& \acclen{63.16}{10.91k}
& \acclen{63.05}{10.92k}
& \acclen{64.12}{10.89k}
& \acclen{\underline{65.34}}{2.51k}
& \acclen{\textbf{66.54}}{2.65k}
& \impinline{\accplus{1.20}}{\lenmore{5.7\%}} \\

\midrule
\rowcolor{SectionGray}
\multicolumn{8}{c}{\textbf{Medical Reasoning}} \\
\midrule
\textbf{Benchmark} 
& \textbf{Base} & \textbf{OPD} & \textbf{ExOPD} & \textbf{EOPD} & \textbf{VETO} 
& \cellcolor{OursTint}\textbf{\method} & \textbf{Imp.} \\
\midrule
PubMed
& \acclen{44.50}{3.53k}
& \acclen{73.00}{8.00k}
& \acclen{\underline{74.10}}{8.19k}
& \acclen{72.10}{8.19k}
& \acclen{73.20}{0.42k}
& \acclen{\textbf{74.80}}{0.43k}
& \impinline{\accplus{0.70}}{\lensave{94.7\%}} \\

MedQA
& \acclen{55.46}{1.28k}
& \acclen{72.58}{7.99k}
& \acclen{73.13}{8.19k}
& \acclen{71.72}{8.19k}
& \acclen{\underline{73.45}}{1.02k}
& \acclen{\textbf{74.86}}{1.01k}
& \impinline{\accplus{1.41}}{\lensave{0.8\%}} \\

Lancet
& \acclen{58.01}{1.22k}
& \acclen{61.41}{8.00k}
& \acclen{\underline{63.83}}{8.19k}
& \acclen{60.44}{8.19k}
& \acclen{63.59}{0.96k}
& \acclen{\textbf{64.56}}{0.92k}
& \impinline{\accplus{0.73}}{\lensave{88.8\%}} \\

MedX
& \acclen{10.35}{3.86k}
& \acclen{\underline{16.01}}{8.00k}
& \acclen{15.73}{8.19k}
& \acclen{\underline{16.01}}{8.19k}
& \acclen{14.77}{1.26k}
& \acclen{\textbf{16.29}}{1.20k}
& \impinline{\accplus{0.28}}{\lensave{85.0\%}} \\

NEJM
& \acclen{55.56}{0.88k}
& \acclen{62.19}{8.00k}
& \acclen{63.35}{8.19k}
& \acclen{63.35}{8.17k}
& \acclen{\underline{63.52}}{0.96k}
& \acclen{\textbf{65.34}}{0.84k}
& \impinline{\accplus{1.82}}{\lensave{12.6\%}} \\

\midrule
\textbf{Avg.}
& \acclen{44.78}{2.15k}
& \acclen{57.04}{8.00k}
& \acclen{\underline{58.03}}{8.19k}
& \acclen{56.72}{8.19k}
& \acclen{57.71}{0.92k}
& \acclen{\textbf{59.17}}{0.88k}
& \impinline{\accplus{1.14}}{\lensave{89.3\%}} \\

\bottomrule
\end{tabular}
\vspace{-0.1in}
\label{tab:knowledge_medical_results_acc_len_imp}
\end{table*}

\begin{figure}[t]
\centering
\footnotesize
\setlength{\fboxsep}{4pt}
\fcolorbox{black!15}{CaseGray}{%
\begin{minipage}{0.95\linewidth}
\textbf{Prompt.} How many down quarks does a tritium ($H^3$) nucleus contain?
\quad \textbf{Ground truth.} $\boxed{5}$

\vspace{0.25em}
\setlength{\tabcolsep}{3pt}
\renewcommand{\arraystretch}{0.98}
\begin{tabularx}{\linewidth}{@{}l c c Y@{}}
\toprule
Method & Outcome & Length & Raw output excerpt \\
\midrule
\method (Ours)
& \correctbadge
& \lenbadge{319 tokens}
& \snippetcell{%
    \rawline{... Total down quarks = 1 + 2 + 2 = **\goodhl{5 down quarks}**}\par\vspace{0.05em}
    \rawline{... Final answer \quad \textdollar{}\textdollar{} \goodhl{\textbackslash boxed\{5\}} \textdollar{}\textdollar{}}
  } \\
\midrule

OPD
& \correctbadge
& \limitbadge
& \snippetcell{%
    \rawline{... Final Answer: \quad \goodhl{\textbackslash boxed\{5\}}}\par\vspace{0.05em}
    \rawline{\patternhl{Final Answer: \textdollar{}\textbackslash boxed\{5\}\textdollar{}} ... \patternhl{Final Answer: \textdollar{}\textbackslash boxed\{5\}\textdollar{}} ... \patternhl{Final Answer: \textdollar{}\rawopenboxed}}
  } \\
\midrule

EOPD
& \wrongbadge
& \limitbadge
& \snippetcell{%
    \rawline{... So **\wronghl{we must reject the problem}**.}\par\vspace{0.05em}
    \rawline{... the \wronghl{problem has a typo} -- or perhaps the **tritium nucleus has charge +}
  } \\
\bottomrule
\end{tabularx}
\end{minipage}}
\vspace{-0.4em}
\caption{\textbf{Case study} on Minerva Math (id 141). We show selected snippets from the raw model outputs. OPD is counted as correct but reaches our 16.4k-token generation limit with repeated final-answer patterns. EOPD also reaches this limit and fails with degenerate spurious behavior. \goodhl{Green}, \patternhl{amber}, and \wronghl{red} indicate correct spans, repetition, and erroneous degeneration, respectively.}
\vspace{-0.18in}
\label{fig:case-tritium}
\end{figure}

\paragraph{Generalization beyond math.}
To validate whether \method generalizes beyond math reasoning, we further
evaluate it on 3 knowledge QA and 5 medical reasoning benchmarks. As shown in
Table~\ref{tab:knowledge_medical_results_acc_len_imp}, on Qwen3-4B-Base, our method achieves the best accuracy on all tasks, improving the strongest average baseline from $65.34$ to $66.54$ on knowledge QA and from $58.03$
to $59.17$ on medical reasoning. The same length-accuracy pattern also
persists in these domains: compared with OPD, the average generation length
drops from $10.91$k to $2.65$k on knowledge QA and from $8.00$k to $0.88$k on
medical reasoning, while accuracy improves in both cases. These results indicate that \method generalizes beyond mathematical reasoning to other domains.

\subsection{Training Efficiency}
\begin{figure}[t]
    \centering
    \begin{minipage}[t]{0.6\linewidth}
        \centering
        \vspace{0pt}
        \includegraphics[width=\linewidth]{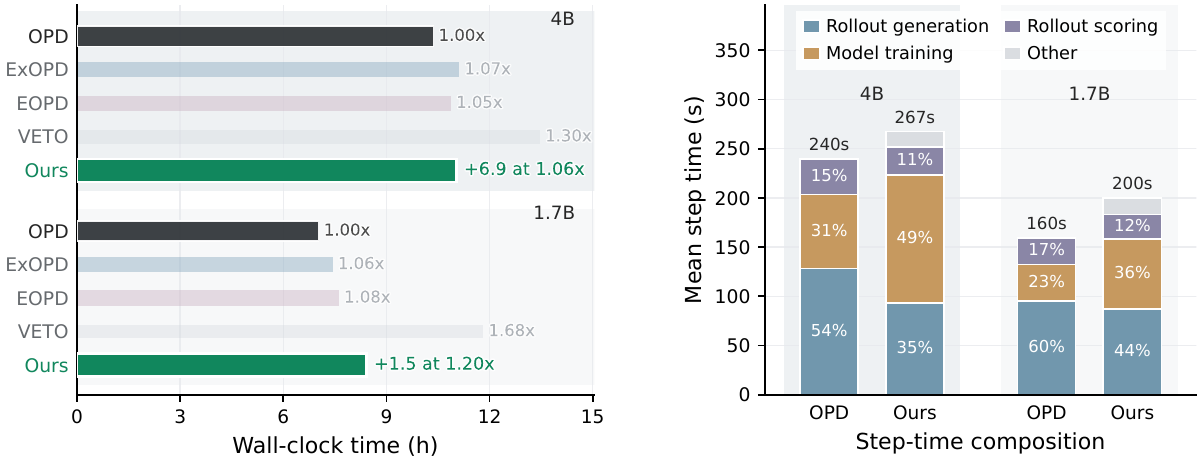}
        \captionsetup{font=footnotesize, skip=3pt}
        \caption{\textbf{Training efficiency comparison.}
        Left: wall-clock training time across methods. Right: step-time breakdown into rollout generation, model training, rollout scoring, and other overhead.}
        \label{fig:training-time-efficiency}
    \end{minipage}
    \hfill
    \begin{minipage}[t]{0.36\linewidth}
        \centering
        \vspace{0pt}
        \scriptsize
        \setlength{\tabcolsep}{3.2pt}
        \renewcommand{\arraystretch}{1.02}
        \captionsetup{font=footnotesize, skip=3pt}
        \captionof{table}{Comparison with OPD using extra student updates on Qwen3-4B-Base math benchmarks. $K$ denotes the number of model update steps performed in each training round.}
        \begin{tabular}{@{}lccc@{}}
            \toprule
            \textbf{Metric} 
            & \textbf{OPD} 
            & \textbf{OPD} 
            & \textbf{Ours} \\
            & \textbf{($K{=}1$)}
            & \textbf{($K{=}2$)}
            & \\
            \midrule
            Avg. Perf. (\%) 
            & 28.46 & 28.59 & 35.37 \\
            Rollout Gen. (s) 
            & 128.37 & 130.58 & 94.68 \\
            Training (s) 
            & 75.00 & 150.01 & 131.12 \\
            Peak Mem. (GB) 
            & 123.65 & 124.04 & 141.99 \\
            \bottomrule
        \end{tabular}
        \label{tab:training-efficiency}
    \end{minipage}
    \vspace{-1.3em}
\end{figure}
Since our method trains an additional proposal policy, we compare its training
cost with existing OPD variants. Fig.~\ref{fig:training-time-efficiency}
reports the wall-clock efficiency and step-time breakdown. Although our method
introduces an additional proposal policy update, it maintains a favorable
accuracy-time trade-off throughout training, with only about $13\%$ extra
wall-clock time on average. Note that VETO is trained for one extra epoch to reach its optimal performance, which accounts for its longer wall-clock time. The step-time breakdown further shows where the extra cost comes from. Our method spends more time on model training due to the additional proposal-policy update, but it partially offsets this cost through faster rollout generation,
as it samples fewer degenerate long responses.

Table~\ref{tab:training-efficiency} further examines whether the improvement is simply due to extra optimization compute. It reports the average performance on math benchmarks along with average per-round rollout generation time, model training time, and peak memory, comparing our method with OPD using one or two student update steps per round. Doubling the number of OPD student updates per round nearly doubles the training time from $75.00$s to $150.01$s, but only changes average accuracy from $28.46\%$ to $28.59\%$. In contrast, our method reaches $35.37\%$ average accuracy with $131.12$s of training time per round, lower than OPD with two updates. The additional proposal policy adds only modest memory overhead ($141.99$ vs.~$\sim$$124$GB) and is partly offset by faster rollout generation ($94.68$ vs.~$\sim$$130$s) since \method generates shorter responses. Together, these numbers indicate that the gain comes from improving the rollout distribution rather than merely spending more compute on student updates.
\begin{figure}[t]
    \centering
    \includegraphics[width=0.98\linewidth]{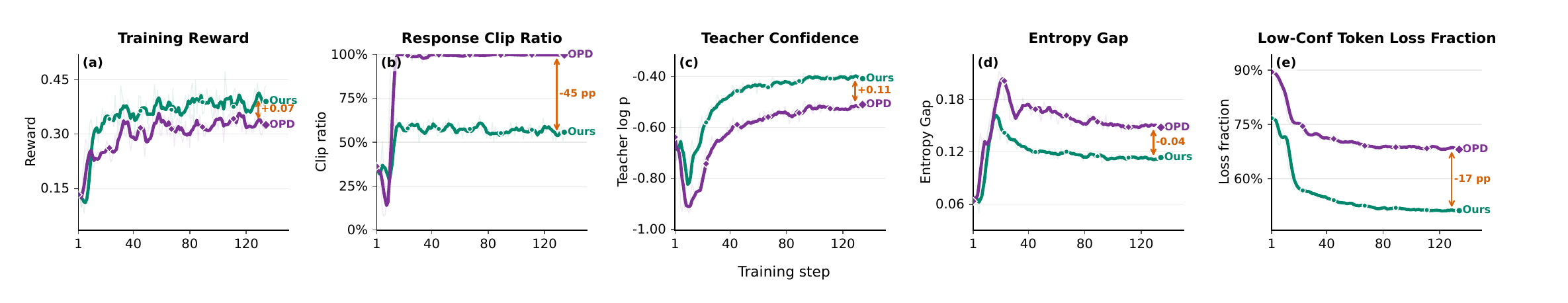}
    \caption{\textbf{Training diagnostics.}
    We compare standard OPD and our tilted rollout proposal over training.
    (a) Training reward. (b) Fraction of responses clipped by the maximum
    generation length. (c) Average teacher log-probability on generated tokens.
    (d) Student--teacher entropy gap. (e) Fraction of token-level loss from
    low-confidence teacher positions.}
    \label{fig:diagnostics}
    \vspace{-1.0em}
\end{figure}
\subsection{Analysis}
\label{sec:analysis}

\textbf{Training Diagnostics.} Fig.~\ref{fig:diagnostics} provides diagnostics for the mechanism behind the
improvements. Panels (a) and (b) first rule out the explanation that our method
only shortens generations by suppressing reasoning. Our method obtains a higher
training reward than OPD, while reducing the fraction of responses that reach
the maximum length by 45 points. Thus, the shorter outputs mainly come from fewer degraded responses hitting the maximum length, rather than from suppressing useful reasoning.
\pagebreak
\begin{wraptable}{r}{0.41\linewidth}
    \vspace{-0.1em}
    \centering
    \small
    \setlength{\tabcolsep}{4.2pt}
    \renewcommand{\arraystretch}{1.05}
    \caption{Effect of the coefficient $\eta$ on selected Qwen3-4B-Base math benchmarks. Scores are reported as accuracy (\%).}
    \begin{tabular}{lcccc}
        \toprule
        $\eta$ & AIME26 & H25N & Minerva & Avg. \\
        \midrule
        0.10 & 23.33 & 15.10 & 42.36 & 26.93 \\
        0.15 & 24.17 & \textbf{15.42} & \textbf{42.65} & 27.41 \\
        0.20 & \textbf{25.00} & \textbf{15.42} & \textbf{42.65} & \textbf{27.69} \\
        0.30 & \underline{24.69} & 15.31 & 42.43 & \underline{27.48} \\
        \bottomrule
    \end{tabular}
    \label{tab:eta-sweep}
    \vspace{-1.0em}
\end{wraptable}
Panels (c)--(e) further show that \method changes the trajectory
distribution used for student supervision. The generated trajectories receive
higher teacher log-probability ($+0.11$), indicating that they remain more
likely under the teacher policy. Following \cite{li2026rethinking}, we examine the student--teacher entropy gap, which shrinks by $0.04$ under our method, indicating that the student's predictive confidence is more aligned with the teacher's. In (e), the fraction of loss from low-confidence teacher positions drops by 17 points, indicating that student updates are increasingly driven by confident teacher supervision. Together with
Fig.~\ref{fig:degradation}, these results clarify the failure mode of
OPD: as rollouts drift away from trajectories likely under the teacher, training
places more loss on low-confidence teacher tokens. The tilted proposal reduces
this pattern, making student updates less dominated by such degraded rollouts.

\paragraph{Impact of the Tilt Strength $\eta$.}
The tilt strength $\eta$ is a hyperparameter that controls how strongly the
rollout policy deviates from the current student toward the teacher. A small $\eta$ keeps the rollout close to the student, while a large $\eta$
weakens on-policy training. As shown in Table~\ref{tab:eta-sweep}, average
accuracy varies within about $0.76$ across the range, peaking
at $\eta=0.20$, indicating that \method is not highly sensitive to this
hyperparameter. Appendix~\ref{app:additional_topk} further varies the teacher support size.
\section{Related Work}
\label{sec:related}

\textbf{Reasoning in Large Language Models.}\quad
Recent progress on LLM reasoning~\citep{ouyang2022training,zelikman2022star,zhao2026recontextrecursiveevidencereplay}
has been driven by two closely related post-training paradigms. Reinforcement learning with verifiable rewards optimizes the policy against
sparse outcome-level feedback, and has emerged as the dominant recipe behind
recent reasoning-focused
models~\citep{shao2024deepseekmath,yu2025dapo,zheng2025group,ren2025riskporiskbasedpolicyoptimization,guo2025deepseek,team2026kimi},
yielding strong gains on mathematical, coding, and agentic reasoning~\cite{dong2025agentic,zou2025latent,xu2026prune,zou2026recursive,zou2026hybrid}.
Distillation instead supervises the student with dense token-level signals
from a stronger
teacher~\citep{hinton2002training,sanh2019distilbert,kim2016sequence}, and
has become a core component of reasoning-oriented post-training
pipelines~\citep{coreteam2025mimounlockingreasoningpotential,deepseekai2026deepseekv4,team2026qwen3},
often paired with distilled reasoning
corpora~\citep{guha2025openthoughts,zhao20251,lozhkov2025openr1math220k,ahmad2025opencodereasoning}
or used to transfer reasoning across scales and
modalities~\citep{bercovich2025llama,peng2025skywork}.

\textbf{On-Policy Distillation.}\quad
On-policy distillation lets the student generate its own trajectories and
receive token-level teacher
supervision~\citep{agarwal2024onpolicy,gu2024minillm,lu2025onpolicydistillation},
addressing the train--test mismatch of off-policy
distillation~\citep{ko2024distillm,gudibande2023false,wen2023f,lin2020autoregressive}.
Owing to its efficiency in transferring strong capabilities from large
teachers into mid-scale students, OPD has been explored along several
directions, including cross-family
distillation~\citep{patiño2025_unlocking_on_policy_distillation_for_any_model_family},
multi-teacher guidance
aggregation~\citep{gu2026co,xiao2026mimo}, black-box
distillation~\citep{ye2025black}, and self-distillation schemes that
internalize the model's in-context capabilities into its
parameters~\citep{yang2025distilling,he2026self,zhao2026self,hubotter2026reinforcement}.
Despite this progress, recent studies have reported failure modes such as
repetition, overlong generations, and entropy
collapse~\citep{fu2026revisiting,jin2026entropy,ko2026scaling,li2026rethinking}.
Existing remedies mainly modify the supervision target, via
alternative divergences~\citep{agarwal2024onpolicy}, extrapolative or
entropy-aware objectives~\citep{yang2026learning,jin2026entropy}, and target
reformulations~\citep{jang2026stable,ko2026scaling}, or adjust the
\emph{rollout policy} through heuristic teacher--student
interleaving~\citep{gu2024minillm,xuspeculative}. In this work, we instead study the rollout policy itself in OPD.

\section{Conclusion}
\label{sec:conclusion}
We study OPD and identify a compounding teacher--student mismatch: student
rollouts can drift into low teacher-reachability regions, where noisy
supervision reinforces degenerate patterns across rounds. To address this, we introduce \method, which formulates the rollout policy
as a local trust-region problem under relative teacher support, yielding a
geometric mixture of the student and teacher. We then approximate this
mixture with a learnable proposal policy that generates rollouts in a single forward pass. Across 16 math, knowledge, and medical reasoning benchmarks, \method
consistently outperforms OPD and recent variants.
Diagnostics show that these gains come from shifting rollouts toward
higher teacher-reachable trajectories, suggesting that the rollout policy
is a critical design choice for improving OPD.
\section*{Acknowledgments}
This work is supported by National Science Foundation under Award No. IIS-2117902. The views and conclusions are those of the authors and should not be interpreted as representing the official policies of the funding agencies or the government.
This research used the DeltaAI advanced computing and data resource, which is supported by the National Science Foundation (award OAC 2320345) and the State of Illinois. DeltaAI is a joint effort of the University of Illinois Urbana-Champaign and its National Center for Supercomputing Applications.
This work used the DeltaAI system at the National Center for Supercomputing Applications through allocation CIS250246  from the Advanced Cyberinfrastructure Coordination Ecosystem: Services \& Support (ACCESS) program, which is supported by National Science Foundation grants \#2138259, \#2138286, \#2138307, \#2137603, and \#2138296.

\clearpage
\newpage
\bibliographystyle{plain}
\bibliography{ref}

\newpage
\appendix
\part*{Appendix}
\addcontentsline{toc}{part}{Appendix}

\renewcommand{\contentsname}{Appendix Contents}
\etocsetnexttocdepth{subsection}
{\hypersetup{linkcolor=[HTML]{800000}}
\localtableofcontents
}
\newpage
\section{More Experiment Details}

\subsection{Training and Evaluation Details}
\label{app:details_evaluation}
\subsubsection{Training}

\paragraph{General setup.} Training is conducted on
2~GPUs with tensor-parallel rollout size 2, rollout count $n=1$, rollout
temperature $0.8$, and rollout top-$p$ $1.0$. The maximum prompt length is
$2048$ tokens and the maximum training response length is $4096$ tokens. We
use a learning rate of $3\!\times\!10^{-6}$ with weight decay $0.1$. 

\paragraph{Domain-specific settings.} For math reasoning and knowledge QA, we
use a training batch size of $256$ with PPO mini-batch size $128$, training on
the full split of each dataset. For medical reasoning, we use a training batch
size of $128$ with PPO mini-batch size $64$, training on a $5\text{K}$ subset.
A summary is provided in Table~\ref{tab:training-config}.

\begin{table}[h]
\centering
\caption{Training configuration across domains. Settings shared across all
domains are marked with ``--''.}
\label{tab:training-config}
\begin{tabular}{lccc}
\toprule
Setting & Math & Knowledge QA & Medical \\
\midrule
Batch size              & 256 & 256 & 128 \\
PPO mini-batch          & 128 & 128 & 64  \\
Training data           & full split & full split & 5K subset \\
\midrule
Learning rate           & \multicolumn{3}{c}{$3\!\times\!10^{-6}$} \\
Weight decay            & \multicolumn{3}{c}{$0.1$} \\
Train temperature       & \multicolumn{3}{c}{$0.8$} \\
Train top-$p$           & \multicolumn{3}{c}{$1.0$} \\
Train max response len. & \multicolumn{3}{c}{$4096$} \\
\bottomrule
\end{tabular}
\end{table}

\subsubsection{Evaluation}
\label{app:evaluation}

\paragraph{Math reasoning.} We use the prompt template
\begin{quote}
\texttt{\{problem\} Please reason step by step, and put your final answer
within \textbackslash boxed\{\}.}
\end{quote}
Completions are sampled with temperature $0.8$, top-$p$ $0.9$, and
\texttt{max\_tokens}$=16384$. We draw $n=32$ completions per problem, except
on OlympiadBench where $n=8$ due to its larger size. We report accuracy and average generated length.

\paragraph{Knowledge QA.} For zero-shot MCQA on GPQA-Diamond and
ARC-Challenge, we use the prompt template
\begin{quote}\small
\texttt{Please solve the following multiple-choice question. Think step by
step. End your response with a single final answer line in one of these exact
formats: Final answer: C \{\textquotedbl answer\textquotedbl:
\textquotedbl C\textquotedbl\}. Use only the choice letter, and do not add any
text after the final answer line. Question: \{question\} Options:
\{options\}}
\end{quote}
For MMLU-Pro, we use subject-specific few-shot exemplars from the validation
split and prompt the model to end with \texttt{The answer is (X)}.
For ARC-Challenge and MMLU-Pro, we randomly sample $500$ questions from each
test set for evaluation, while GPQA-Diamond is evaluated on its full test set.
Completions are sampled with temperature $0.8$ and top-$p$ $0.9$, with $n=8$
per question. We use \texttt{max\_tokens}$=16384$ for GPQA-Diamond and
$8192$ for ARC-Challenge and MMLU-Pro.

\paragraph{Medical reasoning.} We use the prompt template
\begin{quote}\small
\texttt{Question: \{question\} Options: \{options\} Think step by step.
Please show your choice in the answer field with only the choice letter, e.g.,
\{\textquotedbl answer\textquotedbl: \textquotedbl C\textquotedbl\}.}
\end{quote}
We use deterministic MCQA-style decoding with temperature $0.0$, top-$p$ $1.0$,
\texttt{max\_tokens}$=8192$, and $n=1$. We report accuracy and average
generated length.

\subsection{Implementation Details}
\label{app:details_implementation}
\paragraph{Rollouts and teacher queries.} At each round $t$, training
trajectories are sampled from the proposal policy $\hat q_t$ (one rollout per
prompt). The teacher $\pi^\star$ is never used to generate trajectories at
training time; it is queried only to obtain token-level reference
probabilities, which are used both to construct the student distillation loss
and to materialize the tilted target $q^*_{\eta,t+1}$ for the proposal
update. To keep teacher query costs comparable across methods, both the
distillation KL and the construction of $q^*_{\eta,t+1}$ are restricted
to the teacher top-$k$ support ($k=16$).

\paragraph{Student and proposal updates.} Each round consists of one student
gradient step on the rollout batch using the clipped importance weights
$w_{t,j}(\tau)$ in Eq.~\eqref{eq:clipped_IS-_ratio} with clipping threshold $C = 5$,
followed by $K=1$ inner gradient steps of the proposal update toward the
freshly constructed target $q^*_{\eta,t+1}$. The proposal $\hat q_0$ is initialized from $\pi_0$, and
at each subsequent round $\hat q_{t+1}$ is initialized from the updated
student $\pi_{t+1}$ before its inner updates.

\paragraph{Baseline parity.} All baselines share our rollout backend,
optimizer, chat template, training data, and evaluation infrastructure with
\method; we only vary method-specific objective terms and their coefficients.

\newpage
\section{Additional Experimental Results}
\label{app:additional_experimental_results}

\subsection{A Closer Look at Accuracy and Generation Length}
\label{app:additional_decoding}

The main results show that \method improves accuracy while reducing
generation length. These aggregate improvements leave open where the
reduction in length occurs and how it relates to the accuracy gains.
We therefore examine three successive questions, corresponding to panels
(a)--(c) of Table~\ref{tab:additional_decoding}: whether the gains depend
on the decoding budget, whether they are explained by termination
behavior, and where the length reduction occurs.

Following the main setting of our paper, we investigate distilling
Qwen3-4B-Base (student) from Qwen3-4B-Instruct-2507 (teacher) on math
reasoning, trained on DAPO-Math-17K with training rollouts capped at 4K
tokens. We compare \method with OPD, the standard on-policy baseline it
builds on, and VETO, the strongest baseline in this setting
(Table~\ref{tab:main}), which applies the same geometric mixture as the
distillation target rather than as the rollout policy. We also include the
teacher as a reference for accuracy and correct-response length.
All evaluations follow the protocol in
Appendix~\ref{app:details_evaluation}.

\begin{table}[!htbp]
\centering
\small
\renewcommand{\arraystretch}{1.08}
\caption{\textbf{Decoding budgets and termination behavior.}
Results are macro-averaged over eight math benchmarks. Accuracy and rates
are percentages; K denotes thousands of tokens. In (a), the best and
second-best student accuracies are bolded and underlined, respectively.}
\label{tab:additional_decoding}
\begin{tabular*}{\linewidth}{@{\extracolsep{\fill}}lrrr@{}}
\toprule
\rowcolor{SectionGray}
\multicolumn{4}{c}{\textbf{(a) Accuracy at different decoding budgets}}\\
\midrule
Method & Acc@4K & Acc@8K & Acc@16K\\
\midrule
Teacher & 31.60 & 46.56 & 50.16\\
OPD & 21.89 & 25.32 & 28.46\\
VETO & \underline{24.52} & \underline{28.29} & \underline{31.06}\\
\method & \textbf{26.89} & \textbf{31.42} & \textbf{35.37}\\
\bottomrule
\end{tabular*}

\vspace{4pt}
\begin{tabular*}{\linewidth}{@{\extracolsep{\fill}}lrrrr@{}}
\toprule
\rowcolor{SectionGray}
\multicolumn{5}{c}{\textbf{(b) Termination and answer presence at 16K}}\\
\midrule
Method & Term. rate & Acc.\,$\mid$\,term.
& Ans. presence & First ans. pos.\\
\midrule
OPD & 0.60 & 66.86 & 59.65 & 4.85K\\
VETO & 53.27 & 55.13 & 56.17 & 4.68K\\
\method & 55.09 & 61.25 & 59.18 & 4.79K\\
\bottomrule
\end{tabular*}

\vspace{4pt}
\begin{tabular*}{\linewidth}{@{\extracolsep{\fill}}lrrr@{}}
\toprule
\rowcolor{SectionGray}
\multicolumn{4}{c}{\textbf{(c) No-answer termination and correct-response lengths}}\\
\midrule
Method & No-ans. term.
& \shortstack{\# Tokens to\\first correct ans.}
& \shortstack{\# Tokens after\\first correct ans.}\\
\midrule
Teacher & 0.04 & 3.61K & 0.22K\\
OPD & 0.03 & 3.87K & 12.38K\\
VETO & 0.04 & 3.55K & 0.17K\\
\method & 0.01 & 3.75K & 0.34K\\
\bottomrule
\end{tabular*}
\end{table}

\paragraph{(a) Do the gains depend on the decoding budget?}
Training rollouts are capped at 4K tokens, whereas the main evaluation
allows up to 16K, so part of the gap could reflect how each method
extrapolates beyond its training length. We therefore evaluate all methods
under matched decoding budgets, where Acc@4K, Acc@8K, and Acc@16K denote
accuracy with a maximum generation length of 4K, 8K, and 16K tokens.
\method outperforms OPD and VETO at all three budgets, with gains of 5.00
and 2.37  already present at the training-matched 4K
budget. The improvement therefore cannot be attributed solely to
extrapolation beyond the training length.

\paragraph{(b) Is the gap at 16K explained by termination?}
Since several baselines frequently reach the 16K limit, we next examine
termination behavior at this budget. The termination rate (Term.\ rate)
is the fraction of responses that end before the limit, and
Acc.\,$\mid$\,term.\ is the accuracy among these terminated responses.
Answer presence (Ans.\ presence) is the fraction of responses containing
at least one complete boxed answer, whether or not they terminate, and
first-answer position (First ans.\ pos.) is the average token position of
the first complete boxed answer over answer-containing responses,
regardless of correctness. \method and VETO have similar termination
rates (55.09\% and 53.27\%), yet \method is more accurate both overall
(35.37\% vs.\ 31.06\%) and among terminated responses (61.25\% vs.\
55.13\%), so its advantage over VETO is not explained by terminating more
often. OPD and \method have similar answer presence and first-answer
positions, making it unlikely that OPD's lower accuracy arises from
correct answers appearing later and being truncated. OPD's higher
accuracy conditional on termination is measured on only 0.60\% of its
responses and is therefore unstable.

\paragraph{(c) Where does the length reduction come from?}
With termination ruled out as the main explanation, we ask whether the
shorter generations of \method reflect premature stopping or the removal
of redundant continuation. The no-answer termination rate
(No-ans.\ term.) is the fraction of responses that end before the limit
without a parseable final answer, which would rise if the model stopped
prematurely. For the two length columns, we split each correct response
at the end of its first correct boxed answer and report the average
number of tokens before and after this point. Although \method terminates
on 55.09\% of responses, only 0.01\% end without a parseable answer, and
its tokens to the first correct answer (3.75K) are comparable to those of
OPD (3.87K) and the teacher (3.61K). In contrast, tokens generated after
the first correct answer drop from 12.38K for OPD to 0.34K. On every
benchmark, at least 94.6\% of the correct-response length reduction
relative to OPD occurs after the first correct answer, and 95.78\% of
correct OPD responses repeat the already scored boxed answer. These
results indicate that the shorter generations come from removing
post-answer repetition rather than from truncating useful reasoning.
\FloatBarrier

\subsection{Proposal Approximation and Effective KL Radius}
\label{app:additional_proposal}

Our method uses a learned proposal policy to approximate the tilted target.
We examine this approximation after the proposal update at different training
steps in the same 4B student/teacher setting on math benchmarks. For readability, we omit
prefix and round indices and write $\pi$
for the student, $q^*\propto\pi^{1-\eta}(\pi^\star)^\eta$ for the tilted
target, and $\hat q$ for the learned proposal. We measure both KL
directions between $q^*$ and $\hat q$ on the renormalized teacher top-16
support. We also compute Spearman's $\rho$ across prefixes between
$\mathrm{KL}(q^*\Vert\pi)$ and $\mathrm{KL}(\hat q\Vert\pi)$, using
both the top-16 support and full-vocabulary. This correlation
measures whether the target and proposal induce a similar ordering of
displacement from the student across prefixes.

\begin{table}[!htbp]
\centering
\small
\renewcommand{\arraystretch}{1.10}
\caption{\textbf{Proposal approximation and effective KL radius.}
(a) Mean approximation KLs on the teacher top-16 support and rank
correlations of displacement from the student.
(b) Percentiles of $\mathrm{KL}(q^*\Vert\pi)$ on the full vocabulary.}
\label{tab:additional_proposal}
\begin{tabular*}{\linewidth}{@{\extracolsep{\fill}}rrrrr@{}}
\toprule
\rowcolor{SectionGray}
\multicolumn{5}{c}{\textbf{(a) Approximation to the tilted target}}\\
\midrule
Step & $\mathrm{KL}(q^*\Vert\hat q)$ & $\mathrm{KL}(\hat q\Vert q^*)$
& \shortstack{$\rho$\\top-16} & \shortstack{$\rho$\\full vocab.}\\
\midrule
20 & 0.012299 & 0.012704 & 0.880 & 0.704\\
40 & 0.009787 & 0.009905 & 0.862 & 0.746\\
60 & 0.009007 & 0.009122 & 0.862 & 0.820\\
80 & 0.009032 & 0.009111 & 0.888 & 0.763\\
100 & 0.008369 & 0.008462 & 0.882 & 0.789\\
120 & 0.008931 & 0.009011 & 0.888 & 0.766\\
134 & 0.008291 & 0.008411 & 0.894 & 0.710\\
\bottomrule
\end{tabular*}

\vspace{5pt}
\begin{tabular*}{\linewidth}{@{\extracolsep{\fill}}rrrr@{}}
\toprule
\rowcolor{SectionGray}
\multicolumn{4}{c}{\textbf{(b) Effective radius of the tilted target}}\\
\midrule
Step & Median & 70th percentile & 80th percentile\\
\midrule
20 & $6.07\times10^{-7}$ & $4.04\times10^{-4}$ & 0.004136\\
60 & $1.21\times10^{-6}$ & 0.001440 & 0.005581\\
100 & $3.10\times10^{-7}$ & $8.19\times10^{-4}$ & 0.005517\\
134 & $1.33\times10^{-7}$ & $1.83\times10^{-4}$ & 0.008065\\
\bottomrule
\end{tabular*}
\end{table}

Table~\ref{tab:additional_proposal}(a) shows that both approximation KLs
remain small across training, and at the final step the full-vocabulary
KLs are also below 0.01, close to the top-16 values.
The displacement correlations are consistently positive at all measured
steps, showing that $\hat q$ tends to move farther from the student where
$q^*$ requests a larger shift, although they do not imply equal
displacement magnitudes. Overall, the small KLs and stable rank
correlations show that $\hat q$ closely follows both the tilted target and
its relative variation across prefixes, including beyond the truncated
top-16 support.
\method samples from the learned $\hat q$ as an efficient proxy, whereas
the exact geometric-mixture and trust-region interpretation applies to the
ideal target $q^*$.

\paragraph{Effective KL radius.}
For a fixed KL radius, Proposition~\ref{prop:geometric-mixture} yields a
prefix-dependent $\eta$. Solving the constrained problem separately at
every prefix would add online computation, so our implementation instead
treats $\eta$ as a global tunable mixing coefficient. Because the student
and teacher distributions vary across prefixes, a global $\eta$ induces a
non-uniform effective KL radius $\mathrm{KL}(q^*\Vert\pi)$. To quantify
this effect, Table~\ref{tab:additional_proposal}(b) reports its median,
70th, and 80th percentiles on the full vocabulary at four training steps.
Across training, the distribution remains concentrated near zero, with
most prefixes receiving only a small shift and a smaller right tail
receiving larger shifts. Thus, although a fixed $\eta$ induces a
non-uniform, prefix-dependent radius, it does not produce a large
deviation from the student for most prefixes, preserving student
anchoring in practice.
\FloatBarrier

\subsection{Scaling to a Larger Teacher and Student}
\label{app:additional_scaling}

To examine whether the improvements extend to larger model scales, we use
Qwen3-30B-A3B-Instruct-2507 as the teacher and Qwen3-4B-Base or
Qwen3-8B-Base as the student. All methods are trained for two epochs with
teacher top-16 support. The evaluation protocol is unchanged.

\begin{figure}[!htbp]
\centering
\includegraphics[width=\linewidth]{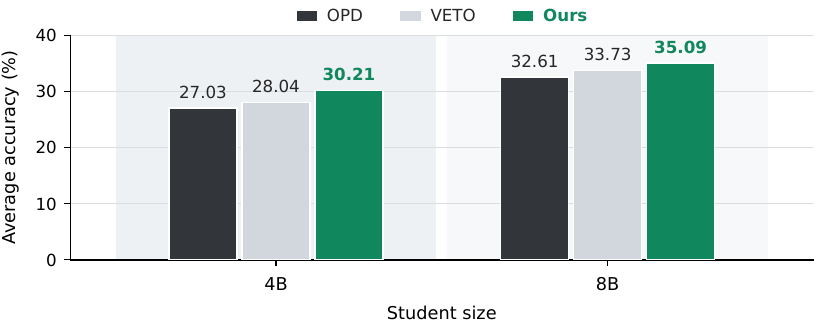}
\caption{\textbf{Distillation with a larger teacher and student.}
Bars show macro-average accuracy (\%) over eight math benchmarks after
two training epochs with Qwen3-30B-A3B-Instruct-2507 as the teacher.
``Ours'' denotes \method.}
\label{fig:additional_scaling}
\end{figure}

Fig.~\ref{fig:additional_scaling} shows that \method achieves the best
average accuracy for both students, improving over OPD by 3.18 and 2.48
points at 4B and 8B, respectively, and over VETO by 2.17 and 1.36 points.
These results indicate that \method remains effective with a larger
teacher and at a larger student scale.
\FloatBarrier

\clearpage
\subsection{Impact of the Teacher Support Size}
\label{app:additional_topk}

We next examine the teacher top-$k$ support used for the student
distillation loss and the construction of the tilted target.
We vary $k\in\{4,16,32,128\}$ in the original 4B student/teacher
setting on math reasoning, with $k=16$ as the default.

\begin{table}[!htbp]
\centering
\footnotesize
\setlength{\tabcolsep}{3.2pt}
\renewcommand{\arraystretch}{1.10}
\caption{\textbf{Impact of the teacher support size.}
Scores are accuracy (\%); Avg.\ is the macro-average over eight math
benchmarks. Best and second-best accuracies are bolded and underlined,
respectively.}
\label{tab:additional_topk}
\begin{tabular*}{\linewidth}{@{\extracolsep{\fill}}rrrrrrrrrr@{}}
\toprule
$k$ & AIME24 & AIME25 & AIME26 & AMC23 & H25F & H25N & Minerva & Olymp. & Avg.\\
\midrule
4 & 27.19 & 27.60 & 23.75 & 66.95 & 14.69 & 15.21 & 42.35 & 55.36 & 34.14\\
16 & 28.65 & \textbf{28.96} & \textbf{25.00} & 69.61 & \textbf{15.83}
& \underline{15.42} & 42.65 & 56.82 & \underline{35.37}\\
32 & \textbf{30.00} & \underline{28.65} & \underline{24.90} & \underline{69.84}
& 14.79 & 14.90 & \underline{42.68} & \underline{56.88} & 35.33\\
128 & \underline{29.48} & \textbf{28.96} & 24.58 & \textbf{70.39}
& \underline{15.62} & \textbf{15.94} & \textbf{42.70} & \textbf{57.60} & \textbf{35.66}\\
\bottomrule
\end{tabular*}
\end{table}

Table~\ref{tab:additional_topk} shows that reducing $k$ from 16 to 4
lowers average accuracy by 1.23. Increasing the support
beyond 16 yields limited changes: $k=16$, 32, and 128 achieve 35.37,
35.33, and 35.66, respectively, a range of 0.33. The gains
are not monotonic across support sizes. These results support top-16
as a practical choice in this setting.
\FloatBarrier

\clearpage
\section{Discussion}
\label{app:discussion}
\paragraph{Why use a learned tilted rollout proposal?}
A possible alternative is to sample from the student \(\pi_t\) and reweight
the resulting trajectories toward the tilted target \(q_{\eta,t}^*\).
However, reweighting only changes the contribution of trajectories that have
already been sampled; it does not change which prefixes and tokens are visited
during rollout generation. If completions with large teacher--student
log-ratio have low probability under \(\pi_t\), they may rarely appear in
student rollouts, and importance weights cannot recover trajectories that were
not collected. The learned proposal \(\hat q_t\) instead changes the rollout
distribution directly by generating trajectories from an approximation to the
tilted target. The clipped weights in Eq.~\eqref{eq:student-update} are used only
to keep the student update aligned with the OPD objective when training on
proposal-generated rollouts.

\paragraph{Implementation considerations.}
We initialize \(\hat q_0 \leftarrow \pi_0\), and at each subsequent round
initialize \(\hat q_{t+1}\) from the updated student \(\pi_{t+1}\) before
training it toward the next tilted target \(q^*_{\eta,t+1}\). This warm-start
keeps the proposal close to the current student at the beginning of each round,
so the proposal update primarily needs to absorb the teacher-induced tilt rather
than re-learn the student distribution from scratch.

Beyond the fixed-\(\eta\) scheme used in our experiments, it is natural to ask
whether \(\eta\) could instead be adapted across rounds, which would correspond
to an implicit curriculum over rollout distributions. One conceivable design is
to set \(\eta\) so that the local KL between \(q^*_{\eta,t+1}\) and
\(\pi_{t+1}\) matches a target radius, and to grow this radius only when the
learned proposal tracks the tilted target stably---plausible diagnostics for
such stability include the proposal-training loss, the proposal--student KL,
and the clipping rate of the importance weights. Conceptually, this would tie
the amount of teacher-guided deviation to the current training state rather
than to a fixed schedule. We flag this as a possible extension rather than a
component of our method, since our reported results use a fixed \(\eta\).

A separate direction concerns the parameterization of \(\hat q_t\) itself.
Rather than maintaining a full copy of the student, one could parameterize
\(\hat q_t\) with a lightweight adapter such as LoRA~\cite{hu2022lora} on top
of the frozen student backbone, in which case the proposal overhead would
reduce to storing and updating a small set of adapter parameters. We expect
this to be largely orthogonal to the rollout-policy design studied here, and
leave a careful empirical comparison to future work.

\paragraph{Limitations.}
Our study has several limitations. The learnable proposal policy \(\hat q_t\)
introduces modest training-time overhead, which may motivate lighter
parameterizations such as LoRA-based proposals at larger student scales. We
also treat the tilt strength \(\eta\) as a fixed scalar; adaptive schedules
that tie \(\eta\) to training-state signals are a natural extension we do not
pursue here. The clipped importance weights \(w_{t,j}(\tau)\)
also introduce a fixed clipping threshold \(C\), and a more principled
treatment of the proposal--student gap could further tighten the alignment
between rollout sampling and the OPD objective. Finally, our empirical study
focuses on the Qwen3 family, with main evaluations at the 1.7B and 4B scales
across math, knowledge QA, and medical reasoning, and additional math
experiments extending to an 8B student in
Appendix~\ref{app:additional_scaling}. We leave extensions to cross-family
teacher--student pairs, still larger scales, and additional modalities to
future work.
\newpage
\section{Declaration of LLM Usage and Broader Impact}
\label{appendix_sec_LLM_usage_broader_impacts}

\textbf{Declaration of LLM Usage.}
This work studies large language model distillation, where LLMs are central
experimental subjects throughout the paper: they are used as teacher and student
policies, sampled to generate reasoning trajectories, queried for token-level
distillation supervision, and evaluated on mathematical, knowledge-intensive,
and medical reasoning benchmarks. LLMs were also used as editing tools to
improve manuscript clarity, grammar, and presentation, without replacing the
authors' role in method design, analysis, or experimental interpretation.

\textbf{Broader Impact.}
This paper introduces \method{}, a rollout-policy design framework for improving
on-policy distillation of reasoning-oriented language models. By guiding
student rollouts toward higher-teacher support trajectories while retaining the
scalability of standalone generation, our method can improve the effectiveness
and budget efficiency of distilling reasoning capabilities into smaller models.
This may make capable reasoning systems more accessible in domains such as
mathematics, knowledge-intensive question answering, and medical reasoning
benchmarks. At the same time, stronger distilled reasoning models may inherit
general risks associated with LLM deployment, including over-reliance on
automated reasoning, hallucinated or incorrect answers, and potential misuse in
settings where generated outputs are treated as authoritative. Our work does not
introduce a new high-risk application setting, collect human-subject data, or
release a model intended for direct deployment; rather, it proposes a training
method for existing reasoning-model distillation pipelines. We therefore view
its broader societal risks as largely aligned with those of existing LLM
reasoning and distillation systems, and encourage downstream users to apply
appropriate verification, monitoring, and domain-specific safeguards before
deployment in high-stakes settings.
\newpage
\section{Missing Proofs}
\label{app:proofs}

\subsection{Proof of Proposition~\ref{prop:geometric-mixture}}
\geomixture*
\label{app:proofs_geometric-mixture}
\begin{proof}
Fix a prefix $(x,y_{<j})$ and write
\[
    p(v)=\pi_t(v\mid x,y_{<j}),
    \qquad
    s(v)=s_t(v,x,y_{<j})
\]
for brevity. The local trust-region problem in Eq.~\eqref{eq:local-trust-region} is
\[
    \max_q \sum_{v\in\mathcal{V}} q(v)s(v)
    \quad
    \mathrm{s.t.}
    \quad
    \mathrm{KL}(q\|p)\leq \delta,
    \qquad
    \sum_{v\in\mathcal{V}}q(v)=1 .
\]
We consider the nondegenerate case in which the KL constraint is active; the degenerate cases are recovered by $\eta=0$ or by a constant score on the student support. The Lagrangian is
\[
    \mathcal{L}(q,\lambda,\nu)
    =
    \sum_v q(v)s(v)
    -
    \lambda
    \left(
        \sum_v q(v)\log\frac{q(v)}{p(v)}
        -
        \delta
    \right)
    -
    \nu
    \left(
        \sum_v q(v)-1
    \right),
\]
where $\lambda>0$ is the multiplier for the KL constraint and $\nu$ enforces normalization. Taking the derivative with respect to $q(v)$ and setting it to zero gives
\[
    s(v)
    -
    \lambda
    \left(
        \log\frac{q(v)}{p(v)}+1
    \right)
    -
    \nu
    =
    0 .
\]
Therefore,
\[
    q(v)
    =
    p(v)
    \exp\left(\frac{s(v)}{\lambda}\right)
    \exp\left(-1-\frac{\nu}{\lambda}\right).
\]
The last factor is independent of $v$ and is absorbed into the normalization constant. Letting $\eta=1/\lambda$, we obtain
\[
    q_\eta^*(v)
    =
    \frac{
        p(v)\exp(\eta s(v))
    }{
        \sum_{v'\in\mathcal{V}}p(v')\exp(\eta s(v'))
    } .
\]
Returning to the original notation yields
\[
    q_\eta^*(v\mid x,y_{<j})
    =
    \frac{
        \pi_t(v\mid x,y_{<j})
        \exp\left(\eta s_t(v,x,y_{<j})\right)
    }{
        Z_\eta(x,y_{<j})
    },
\]
where
\[
    Z_\eta(x,y_{<j})
    =
    \sum_{v'\in\mathcal{V}}
    \pi_t(v'\mid x,y_{<j})
    \exp\left(\eta s_t(v',x,y_{<j})\right).
\]
Substituting
\[
    s_t(v,x,y_{<j})
    =
    \log \pi^\star(v\mid x,y_{<j})
    -
    \log \pi_t(v\mid x,y_{<j})
\]
gives
\[
    \pi_t(v\mid x,y_{<j})
    \exp\left(\eta s_t(v,x,y_{<j})\right)
    =
    \pi_t(v\mid x,y_{<j})^{1-\eta}
    \pi^\star(v\mid x,y_{<j})^\eta .
\]
Hence
\[
    q_\eta^*(v\mid x,y_{<j})
    =
    \frac{
        \pi_t(v\mid x,y_{<j})^{1-\eta}
        \pi^\star(v\mid x,y_{<j})^\eta
    }{
        \sum_{v'\in\mathcal{V}}
        \pi_t(v'\mid x,y_{<j})^{1-\eta}
        \pi^\star(v'\mid x,y_{<j})^\eta
    } .
\]
When the constraint is active, the multiplier $\eta$ is chosen so that 
$\mathrm{KL}(q_\eta^*\|p)=\delta$.
\end{proof}

\subsection{Proof of Corollary~\ref{cor:local-improvement}}
\label{app:proofs_local-improvement}
\localimprovement*

\begin{proof}
Fix a prefix $(x,y_{<j})$ and write
\[
    p(v)=\pi_t(v\mid x,y_{<j}),
    \qquad
    s(v)=s_t(v,x,y_{<j})
\]
for brevity. Let
\[
    \mu=\mathbb{E}_{v\sim p}[s(v)],
    \qquad
    \sigma^2=\mathrm{Var}_{v\sim p}[s(v)] .
\]
By Proposition~\ref{prop:geometric-mixture}, the optimal local tilted distribution can be written as
\[
    q_\eta^*(v)
    =
    \frac{
        p(v)\exp(\eta s(v))
    }{
        Z(\eta)
    },
    \qquad
    Z(\eta)
    =
    \mathbb{E}_{v\sim p}\left[\exp(\eta s(v))\right].
\]
As $\delta\to 0$, the corresponding multiplier satisfies $\eta\to 0$. Define the centered cumulant generating function
\[
    \Lambda(\eta)
    =
    \log
    \mathbb{E}_{v\sim p}
    \left[
        \exp\left(\eta(s(v)-\mu)\right)
    \right].
\]
Then
\[
    \log Z(\eta)
    =
    \eta\mu+\Lambda(\eta),
    \qquad
    S_j(q_\eta^*)
    =
    \frac{d}{d\eta}\log Z(\eta)
    =
    \mu+\Lambda'(\eta).
\]
Since the vocabulary is finite, $\Lambda$ is smooth in a neighborhood of $\eta=0$. Therefore,
\[
    \Lambda(\eta)
    =
    \frac{\sigma^2}{2}\eta^2
    +
    O(\eta^3),
    \qquad
    \Lambda'(\eta)
    =
    \sigma^2\eta
    +
    O(\eta^2).
\]
It follows that
\[
    S_j(q_\eta^*)-S_j(\pi_t)
    =
    S_j(q_\eta^*)-\mu
    =
    \sigma^2\eta
    +
    O(\eta^2).
\]

It remains to relate $\eta$ to the KL radius $\delta$. For the exponential tilt,
\[
    \mathrm{KL}(q_\eta^*\|p)
    =
    \mathbb{E}_{q_\eta^*}
    \left[
        \log\frac{q_\eta^*(v)}{p(v)}
    \right]
    =
    \eta S_j(q_\eta^*)-\log Z(\eta).
\]
Using the identities above,
\[
    \mathrm{KL}(q_\eta^*\|p)
    =
    \eta\bigl(\mu+\Lambda'(\eta)\bigr)
    -
    \bigl(\eta\mu+\Lambda(\eta)\bigr)
    =
    \eta\Lambda'(\eta)-\Lambda(\eta).
\]
Substituting the Taylor expansions gives
\[
    \delta
    =
    \mathrm{KL}(q_\eta^*\|p)
    =
    \frac{\sigma^2}{2}\eta^2
    +
    O(\eta^3).
\]
Thus, when $\sigma>0$,
\[
    \eta
    =
    \frac{\sqrt{2\delta}}{\sigma}
    +
    O(\delta).
\]
Substituting this relation into the improvement expansion yields
\[
    S_j(q_\eta^*)-S_j(\pi_t)
    =
    \sigma^2
    \left(
        \frac{\sqrt{2\delta}}{\sigma}
        +
        O(\delta)
    \right)
    +
    O(\delta)
    =
    \sigma\sqrt{2\delta}
    +
    O(\delta).
\]
Returning to the original notation, $\sigma=\sigma_j(x,y_{<j})$, we obtain
\[
    S_j(q_\eta^*)-S_j(\pi_t)
    =
    \sigma_j(x,y_{<j})\sqrt{2\delta}
    +
    o(\sqrt{\delta}).
\]
If $\sigma=0$, then the score is constant on the support of $\pi_t(\cdot\mid x,y_{<j})$, so any finite-KL reweighting has the same expected score. The first-order improvement is therefore zero, which is consistent with the stated expansion.
\end{proof}

\newpage

\section{Optimization Analysis under Structured Teacher Target Noise}
\label{app:structured-teacher-target-noise}

We analyze a stochastic teacher--student setting. At round \(t\), the current
student and proposal policies may be random through the training history; in the
one-step analysis below, we condition on this history and suppress the
conditioning notation. A token-level state is denoted by
\(c=(x,j,y_{<j})\). For any rollout policy \(q\), let
\[
M_q
:=
\mathbb{E}_{x\sim\mathcal D,\tau\sim q(\cdot\mid x)}
\left[|\tau|\right]
\in(0,\infty).
\]
Define the normalized token-level occupancy distribution \(d^q\) by
\[
\mathbb{E}_{c\sim d^q}[f(c)]
:=
\frac{1}{M_q}
\mathbb{E}_{x\sim\mathcal D,\tau\sim q(\cdot\mid x)}
\left[
\sum_{j=1}^{|\tau|}
f(x,j,y_{<j})
\right].
\]

We introduce an ideal clean teacher target
\(\pi_\circ^\star(\cdot\mid c)\) only for analysis. The teacher target actually
used for distillation is a random conditional distribution
\(\widetilde \pi^\star(\cdot\mid c)\), which models stochastic or uncertain
teacher supervision around the clean target.

To relate teacher reachability to supervision reliability, let \(r_T(c)\ge 0\)
denote a reachability-risk score, with larger values assigned to prefixes that
are less reachable under the teacher. For example, \(r_T(c)\) can be chosen as a
nondecreasing function of the teacher prefix negative log-likelihood,
\[
    m_T(c)
    =
    -\frac{1}{\max\{1,j-1\}}
    \sum_{k=1}^{j-1}
    \log \pi^\star(y_k\mid x,y_{<k}),
    \qquad
    r_T(c)=\phi(m_T(c)),
\]
where \(\phi\) is nondecreasing.

For a state \(c\), define the clean and observed teacher losses as
\[
\ell_c^\circ(\theta)
=
\mathrm{KL}
\left(
\pi_\theta(\cdot\mid c)
\,\|\,
\pi_\circ^\star(\cdot\mid c)
\right),
\qquad
\widetilde \ell_c(\theta)
=
\mathrm{KL}
\left(
\pi_\theta(\cdot\mid c)
\,\|\,
\widetilde \pi^\star(\cdot\mid c)
\right).
\]
The average per-token reachability-risk exposure of a rollout policy \(q\) is
\[
    R_t(q)=\mathbb{E}_{c\sim d^q}[r_T(c)] .
\]

\begin{assumption}[Reachability-structured teacher target noise]
\label{assump:structured-teacher-target-noise}
For each token-level state \(c\), define
\[
\zeta_t(c)
=
\nabla_\theta \widetilde \ell_c(\theta_t)
-
\nabla_\theta \ell_c^\circ(\theta_t).
\]
The perturbation is conditionally unbiased relative to the clean teacher target,
and its conditional second moment is controlled by reachability risk:
\[
\mathbb{E}[\zeta_t(c)\mid c]=0,
\qquad
\mathbb{E}[\|\zeta_t(c)\|^2\mid c]
\le
\lambda r_T(c)
\]
for some \(\lambda>0\).
\end{assumption}

This is an analysis-only noise model: teacher supervision is unbiased relative
to \(\pi_\circ^\star\), but its gradient perturbation is allowed to be larger on
lower-reachability states as measured by \(r_T(c)\).

For the clean target, define the normalized per-token objective
\(F^\circ_{t,q}(\theta)=\mathbb{E}_{c\sim d^q}[\ell_c^\circ(\theta)]\), and let
\[
G_t(q)
=
\mathbb{E}_{c\sim d^q}
\left[
\|\nabla_\theta \ell_c^\circ(\theta_t)\|^2
\right].
\]

\begin{proposition}[One-step descent under reachability-structured teacher target noise]
\label{prop:structured-target-noise-descent}
Assume that \(F^\circ_{t,q}\) is \(L\)-smooth in \(\theta\) for a fixed occupancy
\(d^q\). Consider one stochastic update using the observed teacher target:
\[
\theta_{t+1}
=
\theta_t
-
\alpha
\nabla_\theta \widetilde \ell_c(\theta_t),
\qquad
c\sim d^q .
\]
Under Assumption~\ref{assump:structured-teacher-target-noise},
\[
\mathbb{E}
\left[
F^\circ_{t,q}(\theta_t)-F^\circ_{t,q}(\theta_{t+1})
\right]
\ge
\alpha
\|\nabla_\theta F^\circ_{t,q}(\theta_t)\|^2
-
\frac{L\alpha^2}{2}
\left(
G_t(q)+\lambda R_t(q)
\right),
\]
where the expectation is over \(c\sim d^q\) and the stochastic teacher target.
\end{proposition}

\begin{proof}
Let
\[
    g_c^\circ=\nabla_\theta \ell_c^\circ(\theta_t),
    \qquad
    \widetilde g_c
    =
    \nabla_\theta \widetilde \ell_c(\theta_t)
    =
    g_c^\circ+\zeta_t(c).
\]
By \(L\)-smoothness,
\[
F^\circ_{t,q}(\theta_{t+1})
\le
F^\circ_{t,q}(\theta_t)
-
\alpha
\left\langle
\nabla_\theta F^\circ_{t,q}(\theta_t),
\widetilde g_c
\right\rangle
+
\frac{L\alpha^2}{2}
\|\widetilde g_c\|^2 .
\]
Taking expectation over \(c\sim d^q\) and the observed-target randomness, the
conditional zero-mean property gives
\[
    \mathbb{E}[\widetilde g_c]
    =
    \mathbb{E}_{c\sim d^q}[g_c^\circ]
    =
    \nabla_\theta F^\circ_{t,q}(\theta_t).
\]
Therefore the first-order term becomes
\[
    \mathbb{E}
    \left[
    \left\langle
    \nabla_\theta F^\circ_{t,q}(\theta_t),
    \widetilde g_c
    \right\rangle
    \right]
    =
    \|\nabla_\theta F^\circ_{t,q}(\theta_t)\|^2 .
\]
For the second-order term,
\[
\mathbb{E}\|\widetilde g_c\|^2
=
\mathbb{E}\|g_c^\circ+\zeta_t(c)\|^2 .
\]
The cross term vanishes because
\(\mathbb{E}[\zeta_t(c)\mid c]=0\), and Assumption~\ref{assump:structured-teacher-target-noise}
implies
\[
\mathbb{E}\|\widetilde g_c\|^2
\le
\mathbb{E}_{c\sim d^q}\|g_c^\circ\|^2
+
\lambda \mathbb{E}_{c\sim d^q}[r_T(c)]
=
G_t(q)+\lambda R_t(q).
\]
Substituting these two bounds into the smoothness inequality proves the claim.
\end{proof}

\paragraph{Comparison with on-policy sampling.}
Proposition~\ref{prop:structured-target-noise-descent} assigns each rollout
policy \(q\) the proposal-induced one-step lower bound
\[
\mathcal B_t(q)
=
\alpha
\|\nabla_\theta F^\circ_{t,q}(\theta_t)\|^2
-
\frac{L\alpha^2}{2}
\left(
G_t(q)+\lambda R_t(q)
\right).
\]
This is not a guarantee on a fixed evaluation objective. It compares the descent
bound for the clean objective induced by each rollout occupancy.

To compare the ideal tilted rollout policy \(q^*_{\eta,t}\) with standard
on-policy sampling \(\pi_t\), define
\[
A_t
=
\|\nabla_\theta F^\circ_{t,q^*_{\eta,t}}(\theta_t)\|^2
-
\|\nabla_\theta F^\circ_{t,\pi_t}(\theta_t)\|^2,
\qquad
V_t
=
G_t(q^*_{\eta,t})-G_t(\pi_t),
\qquad
D_t
=
R_t(\pi_t)-R_t(q^*_{\eta,t}).
\]
Then
\[
\mathcal B_t(q^*_{\eta,t})
-
\mathcal B_t(\pi_t)
=
\alpha A_t
-
\frac{L\alpha^2}{2}V_t
+
\frac{L\alpha^2\lambda}{2}D_t .
\]
Therefore, the tilted rollout policy has a no-smaller proposal-induced one-step
lower bound than on-policy sampling whenever
\[
A_t+\frac{\alpha L\lambda}{2}D_t
\ge
\frac{\alpha L}{2}V_t .
\]

This decomposition separates three effects. The term \(A_t\) captures the change
in clean descent signal under the rollout-induced state occupancy. The term
\(V_t\) captures the change in clean-gradient second moment. The term \(D_t\)
captures the reduction in reachability-risk exposure. Thus tilted sampling
improves this proposal-induced lower bound when it increases the clean descent
signal and/or reduces reachability-risk exposure without inducing too large a
second-moment penalty.

\paragraph{When does tilting reduce reachability-risk exposure?}
The comparison above treats
\(D_t=R_t(\pi_t)-R_t(q^*_{\eta,t})\) as a property of the induced rollout
distribution. The geometric tilt does not automatically guarantee \(D_t>0\),
because it is defined through relative teacher support rather than directly
through the reachability-risk score. A local diagnostic can be expressed using a
covariance condition.

For a prefix state \(c\), let
\[
    s_t(v,c)
    =
    \log \pi^\star(v\mid c)-\log \pi_t(v\mid c)
\]
be the teacher--student log-ratio score used by the tilted rollout policy, and let
\(\rho_T(v,c)\) denote the reachability-risk score of the next prefix obtained
by appending token \(v\). For
\(q_\eta(v\mid c)\propto \pi_t(v\mid c)\exp(\eta s_t(v,c))\),
\[
\left.
\frac{d}{d\eta}
\mathbb{E}_{v\sim q_\eta(\cdot\mid c)}
\left[
\rho_T(v,c)
\right]
\right|_{\eta=0}
=
\mathrm{Cov}_{v\sim\pi_t(\cdot\mid c)}
\left(
\rho_T(v,c),
s_t(v,c)
\right).
\]
Therefore, if tokens with larger relative teacher support tend to lead to lower
reachability-risk next prefixes, the covariance is negative and a small
geometric tilt locally reduces reachability-risk exposure. This condition links
the risk-reduction term \(D_t\) in the optimization bound to the proposed
tilted rollout policy.

\newpage
\section{Geometric and Arithmetic Interpolation}
\label{app:geometric-arithmetic-interpolation}

\paragraph{Interpolation paths.}
A natural and commonly used way to interpolate between two policies is
the arithmetic mixture $(1-\alpha)\pi_t + \alpha \pi^\star$~\cite{gu2024minillm,ross2011reduction}. This appendix
contrasts it with the geometric interpolation used in our rollout
policy and clarifies which proposal-design objective each path
corresponds to.

For a fixed prefix \(c=(x,y_{<j})\), write
\[
    p(v)=\pi_t(v\mid c),
    \qquad
    p^\star(v)=\pi^\star(v\mid c)
\]
for the student and teacher conditionals. The geometric interpolation used in
our method is
\[
q_{\eta}^{\mathrm{geo}}(v\mid c)
=
\frac{p(v)^{1-\eta}p^\star(v)^\eta}
{\sum_{v'}p(v')^{1-\eta}p^\star(v')^\eta}.
\]
A natural alternative is the arithmetic interpolation
\[
q_{\alpha}^{\mathrm{arith}}(v\mid c)
=
(1-\alpha)p(v)+\alpha p^\star(v).
\]
Both paths interpolate between the student and teacher endpoints, but they arise
from different information-geometric objectives.
\paragraph{Different variational objectives.}
The geometric path is the barycenter under the reverse-KL geometry:
\[
q_{\eta}^{\mathrm{geo}}
=
\arg\min_q
(1-\eta)\mathrm{KL}(q\|p)
+
\eta\mathrm{KL}(q\|p^\star).
\]
Thus it can be viewed as a product-of-experts~\cite{hinton2002training} interpolation: probability mass is
assigned according to the joint support of the student and teacher conditionals.
This geometry matches the local trust-region formulation used in our rollout policy.
Indeed, with the teacher--student log-ratio
\[
    s_t(v,c)=\log p^\star(v)-\log p(v),
\]
the same distribution can be written as
\[
    q_{\eta}^{\mathrm{geo}}(v\mid c)
    \propto
    p(v)\exp(\eta s_t(v,c)).
\]
Therefore, the geometric rollout policy is the KL-regularized mirror-descent update
that starts from the student conditional and reweights candidate tokens by
relative teacher support.

The arithmetic path corresponds to a different variational problem:
\[
q_{\alpha}^{\mathrm{arith}}
=
\arg\min_q
(1-\alpha)\mathrm{KL}(p\|q)
+
\alpha\mathrm{KL}(p^\star\|q).
\]
This is a mixture-of-experts interpolation: the proposal combines the two
conditionals additively. Such a path is natural when the objective is coverage.
For example, since
\[
    q_{\alpha}^{\mathrm{arith}}(v\mid c)\ge (1-\alpha)p(v),
\]
the importance ratio
\(p(v)/q_{\alpha}^{\mathrm{arith}}(v\mid c)\) is uniformly bounded by
\(1/(1-\alpha)\). This property is useful in settings where the proposal must
cover the student distribution, such as defensive importance sampling or
off-policy evaluation.
\paragraph{Implication for OPD.}
The proposal-design problem in OPD is different. We seek a rollout policy that
remains anchored to the current student while shifting probability toward
teacher supported completions. This is formalized in the main text through a
local \(\mathrm{KL}(q\|\pi_t)\) trust region and the relative teacher support
score \(s_t\). Under this formulation, the geometric path is not merely an
interpolation heuristic; it is the variational solution induced by the chosen
reverse-KL geometry. In contrast, the arithmetic path corresponds to a different
notion of closeness and a different proposal objective, one that emphasizes
additive coverage of the student--teacher mixture.

This distinction is also consistent with the reachability-structured noise
analysis. That analysis does not require a particular interpolation path; it
states that a rollout policy is favorable when its induced occupancy reduces
reachability-risk exposure without incurring too large a second-moment penalty.
The geometric rollout policy is aligned with this analysis because it changes the
student rollout distribution through multiplicative reweighting by relative
teacher support. Whether this reweighting reduces reachability-risk exposure
depends on the covariance condition discussed earlier, not on an unconditional
dominance claim over other interpolation paths.

In summary, arithmetic interpolation is a principled choice for
coverage-oriented objectives and bounded importance ratios. Geometric
interpolation is the appropriate path for the formulation studied in this work:
a student-anchored, reverse-KL-regularized proposal that tilts the current
student toward tokens with larger relative teacher support.
\newpage
\section{More Case Studies}
\newcolumntype{Z}{>{\raggedright\arraybackslash}X}

\definecolor{MBCaseGreen}{HTML}{DFF3E7}
\definecolor{MBCaseGreenText}{HTML}{14532D}
\definecolor{MBCaseRed}{HTML}{FEE2E2}
\definecolor{MBCaseRedText}{HTML}{7F1D1D}
\definecolor{MBCaseBlue}{HTML}{E8F0FE}
\definecolor{MBCaseBlueText}{HTML}{1D4ED8}
\definecolor{MBCaseAmber}{HTML}{FEF3C7}
\definecolor{MBCaseAmberText}{HTML}{92400E}
\definecolor{MBCaseGray}{HTML}{F6F7F9}

\newcommand{\mbcasebadge}[3]{%
  \begingroup\setlength{\fboxsep}{1.5pt}%
  \colorbox{#1}{\strut\scriptsize\textcolor{#2}{\textbf{#3}}}%
  \endgroup%
}
\newcommand{\mbokbadge}{\mbcasebadge{MBCaseGreen}{MBCaseGreenText}{Correct}}
\newcommand{\mbwrongbadge}{\mbcasebadge{MBCaseRed}{MBCaseRedText}{Wrong}}
\newcommand{\mbparsefailbadge}{\mbcasebadge{MBCaseRed}{MBCaseRedText}{Parse fail}}
\newcommand{\mbdriftbadge}{\mbcasebadge{MBCaseRed}{MBCaseRedText}{Answer drift}}
\newcommand{\mboverlongbadge}{\mbcasebadge{MBCaseAmber}{MBCaseAmberText}{Overlong}}
\newcommand{\mbcompactbadge}{\mbcasebadge{MBCaseGreen}{MBCaseGreenText}{Compact}}
\newcommand{\mblenbadge}[1]{\mbcasebadge{MBCaseBlue}{MBCaseBlueText}{#1}}
\newcommand{\mbrawline}[1]{{\ttfamily\scriptsize #1}}
\newcommand{\mbsnippetcell}[2]{%
  \begin{minipage}[t]{\linewidth}\raggedright
  \textbf{Read.} #1\par\vspace{0.12em}
  \textbf{Snippet.} \mbrawline{#2}
  \end{minipage}%
}
\newcommand{\mbcaserow}[5]{#1 & #2 & #3 & \mbsnippetcell{#4}{#5} \\}

\begin{figure*}[ht]
\centering
\footnotesize
\setlength{\fboxsep}{4pt}
\fcolorbox{black!15}{MBCaseGray}{%
\begin{minipage}{0.98\linewidth}
\textbf{Case A: first-order time constant.}
\textbf{Prompt.} For $\dot{y}+5y=10u$, what is the time constant?
\quad \textbf{Ground truth.} $\boxed{1/5}$.

\vspace{0.25em}
\textbf{Takeaway.} Ours identifies the system parameter and stops after a valid boxed answer. The four baselines often reach the same value but continue into long, unparseable tails.

\vspace{0.25em}
\setlength{\tabcolsep}{3pt}
\renewcommand{\arraystretch}{1.04}
\begin{tabularx}{\linewidth}{@{}l c c Z@{}}
\toprule
Method & Outcome & Length & Interpretation and raw-output snippet \\
\midrule
\mbcaserow{Ours}{\mbokbadge}{\mblenbadge{303 tokens}}
{Maps the ODE to $\dot y+a y=bu$, sets $a=5$, and emits the final answer.}
{\(a=5\), so \(\tau=1/a=1/5\) ... Final Answer: \(\boxed{1/5}\)}
\midrule
\mbcaserow{OPD}{\mbparsefailbadge}{\mblenbadge{16.4k tokens}}
{Contains the correct value, then repeats the final-answer pattern until truncation.}
{Final Answer: \(\boxed{0.2}\) ... Final Answer: \(\boxed{0.2}\) ... Final Answer: open boxed answer}
\mbcaserow{ExOPD}{\mbparsefailbadge}{\mblenbadge{16.4k tokens}}
{Starts from the right dynamics but drifts into a 63.2 percent transient discussion.}
{steady-state value is 2 ... \(t=-\frac{1}{5}\ln(0\ldots)\)}
\mbcaserow{EOPD}{\mbparsefailbadge}{\mblenbadge{16.4k tokens}}
{Continues after solving and loses the final answer in a self-contradictory transient calculation.}
{time constant is not the same ... solution is \(y(t)=20(1-e^{-5t})\) ...}
\mbcaserow{VETO}{\mbparsefailbadge}{\mblenbadge{16.4k tokens}}
{Also over-extends the transient interpretation and ends without a parseable answer.}
{\(2(1-e^{-5\tau})=0.736\) ... \(\tau=-\ln(0.632)/5\) ...}
\bottomrule
\end{tabularx}
\end{minipage}}
\caption{Multi-baseline case study on Minerva Math id 128. The qualitative advantage is not just shorter length: Ours preserves answer extractability after the correct reasoning step, while all four baselines shown lose the parseable final answer.}
\label{fig:case-time-constant-all-baselines}
\end{figure*}
\clearpage
\begin{figure*}[ht]
\centering
\footnotesize
\setlength{\fboxsep}{4pt}
\fcolorbox{black!15}{MBCaseGray}{%
\begin{minipage}{0.98\linewidth}
\textbf{Case B: tritium down quarks.}
\textbf{Prompt.} How many down quarks does a tritium ($H^3$) nucleus contain?
\quad \textbf{Ground truth.} $\boxed{5}$.

\vspace{0.25em}
\textbf{Takeaway.} The problem is short and easy to audit. Ours uses the nuclear composition directly; the four baselines drift to an internally inconsistent quark count.

\vspace{0.25em}
\setlength{\tabcolsep}{3pt}
\renewcommand{\arraystretch}{1.04}
\begin{tabularx}{\linewidth}{@{}l c c Z@{}}
\toprule
Method & Outcome & Length & Interpretation and raw-output snippet \\
\midrule
\mbcaserow{Ours}{\mbokbadge}{\mblenbadge{309 tokens}}
{Uses one proton and two neutrons, then counts only down quarks.}
{1 proton \(\to\) 1 down; 2 neutrons \(\to\) 2 + 2 down; total = 5 ... \(\boxed{5}\)}
\midrule
\mbcaserow{OPD}{\mbdriftbadge}{\mblenbadge{16.4k tokens}}
{Repeats an incorrect final count after visiting the right topic.}
{Final Answer: \(\boxed{3}\) quarks ... Final Answer: \(\boxed{3}\) quarks ...}
\mbcaserow{ExOPD}{\mbdriftbadge}{\mblenbadge{16.4k tokens}}
{Falls into the same wrong final count and continues repeating it.}
{Yes. Final answer is: \(\boxed{3}\) down quarks ... Final Answer: \(\boxed{3}\)}
\mbcaserow{EOPD}{\mbdriftbadge}{\mblenbadge{16.4k tokens}}
{Ends in repeated confirmation of the wrong boxed answer.}
{Final Answer: \(\boxed{3}\) ... (Yes.) Final ... (End.) ...}
\mbcaserow{VETO}{\mbdriftbadge}{\mblenbadge{1.9k tokens}}
{Correctly names tritium as one proton and two neutrons, then swaps quark content.}
{1 proton: 1 u and 2 d; 2 neutrons: each has 1 d and 2 u ... total = 3}
\bottomrule
\end{tabularx}
\end{minipage}}
\caption{Multi-baseline case study on Minerva Math id 141. Ours preserves the simple invariant: proton $uud$ contributes one down quark and each neutron $udd$ contributes two. The four baselines drift to $\boxed{3}$.}
\label{fig:case-tritium-all-baselines}
\end{figure*}
\clearpage
\begin{figure*}[ht]
\centering
\footnotesize
\setlength{\fboxsep}{4pt}
\fcolorbox{black!15}{MBCaseGray}{%
\begin{minipage}{0.98\linewidth}
\textbf{Case C: logarithmic equation.}
\textbf{Prompt.} Determine all real $x$ for which $2\log_2(x-1)=1-\log_2(x+2)$.
\quad \textbf{Ground truth.} $\boxed{\sqrt{3}}$.

\vspace{0.25em}
\textbf{Takeaway.} Ours reduces the equation to a cubic with a domain check. The four baselines get trapped in numerical checking or contradiction loops and do not provide a parsed final answer.

\vspace{0.25em}
\setlength{\tabcolsep}{3pt}
\renewcommand{\arraystretch}{1.04}
\begin{tabularx}{\linewidth}{@{}l c c Z@{}}
\toprule
Method & Outcome & Length & Interpretation and raw-output snippet \\
\midrule
\mbcaserow{Ours}{\mbokbadge}{\mblenbadge{785 tokens}}
{Combines logs, solves $(x-1)^2(x+2)=2$, and applies $x>1$.}
{\((x-1)^2(x+2)=2\) \(\Rightarrow x^3-3x=0\); domain \(x>1\) \(\Rightarrow x=\sqrt{3}\); \(\boxed{\sqrt{3}}\)}
\midrule
\mbcaserow{OPD}{\mbparsefailbadge}{\mblenbadge{16.4k tokens}}
{Derives the right transformed equation, then rejects its own solution.}
{\((x-1)^2(x+2)=2\) ... only at \(x=\sqrt{3}\), which is not a solution to ...}
\mbcaserow{ExOPD}{\mbparsefailbadge}{\mblenbadge{16.4k tokens}}
{Switches from symbolic solving to unstable decimal checking.}
{\(\log_2(0.732)\approx -0.448\) ... \(f(1.732)\approx 1.035\) ...}
\mbcaserow{EOPD}{\mbparsefailbadge}{\mblenbadge{16.4k tokens}}
{Keeps recomputing approximations and never returns a parseable final set.}
{\(\ln(0.732)/\ln(2)\approx -0.455\) ... \(\log_2(1.09)\approx\ldots\)}
\mbcaserow{VETO}{\mbparsefailbadge}{\mblenbadge{16.4k tokens}}
{Continues comparing sides numerically instead of closing the algebraic solution.}
{At \(x=2\): \(0>-1\) ... LHS \(>\) RHS ... only if LHS = RHS ...}
\bottomrule
\end{tabularx}
\end{minipage}}
\caption{Multi-baseline case study on OlympiadBench id 2548. Ours follows a short symbolic path and keeps the domain constraint; every one of the four baseline rows shown fails to produce a parsed answer after much longer output.}
\label{fig:case-log-equation-all-baselines}
\end{figure*}
\clearpage
\begin{figure*}[ht]
\centering
\footnotesize
\setlength{\fboxsep}{4pt}
\fcolorbox{black!15}{MBCaseGray}{%
\begin{minipage}{0.98\linewidth}
\textbf{Case D: easy logarithm.}
\textbf{Prompt.} Let $T=16$. Compute $x$ such that $\log_4 T=\log_2 x$.
\quad \textbf{Ground truth.} $\boxed{4}$.

\vspace{0.25em}
\textbf{Takeaway.} This case isolates the length advantage: several OPD-family baselines eventually get the same answer but hit the 16.4k-token budget doing so, while Ours solves and stops.

\vspace{0.25em}
\setlength{\tabcolsep}{3pt}
\renewcommand{\arraystretch}{1.04}
\begin{tabularx}{\linewidth}{@{}l c c Z@{}}
\toprule
Method & Outcome & Length & Interpretation and raw-output snippet \\
\midrule
\mbcaserow{Ours}{\mbokbadge}{\mblenbadge{239 tokens}}
{Substitutes $T=16$, computes $\log_4 16=2$, and solves $x=2^2$.}
{\(\log_4 16=2\); \(\log_2 x=2\); \(x=2^2=4\); Final Answer: \(\boxed{4}\)}
\midrule
\mbcaserow{OPD}{\mboverlongbadge}{\mblenbadge{16.4k tokens}}
{Correct answer, but output continues into repeated confirmation tokens.}
{... \(\boxed{4}\) ... [repeated confirmation tokens continue]}
\mbcaserow{ExOPD}{\mboverlongbadge}{\mblenbadge{16.4k tokens}}
{Correct answer, but continues long after the final answer.}
{... Final Answer: \(\boxed{4}\) ... [repeated confirmation tokens continue]}
\mbcaserow{EOPD}{\mboverlongbadge}{\mblenbadge{16.4k tokens}}
{Same overgeneration pattern despite solving the elementary equation.}
{... Final Answer: \(\boxed{4}\) ... [repeated confirmation tokens continue]}
\mbcaserow{VETO}{\mbokbadge}{\mblenbadge{482 tokens}}
{Correct and comparatively short, but still longer and less direct than Ours.}
{\(\log_2 x=2 \Rightarrow x=2^2=4\) ... \(\boxed{4}\) ... check matches}
\bottomrule
\end{tabularx}
\end{minipage}}
\caption{Multi-baseline case study on OlympiadBench id 2783. Ours gives the shortest correct response among the compared methods; OPD, ExOPD, and EOPD solve the problem but exhaust a 16.4k-token generation budget after the answer.}
\label{fig:case-easy-log-all-baselines}
\end{figure*}

\end{document}